\documentclass[a4paper,12pt]{article}
\usepackage[margin=1in]{geometry}
\usepackage[utf8]{inputenc} 
\usepackage[T1]{fontenc}    
\usepackage{comment}
\usepackage{hyperref}       
\usepackage{url}            
\usepackage{nicefrac}       
\usepackage{microtype}     

\usepackage{xcolor}
\usepackage{framed}
\colorlet{shadecolor}{pink} 
\usepackage{authblk}
\usepackage{adjustbox}
\usepackage{comment}

\usepackage{graphicx}
\usepackage{soul}
\usepackage{subcaption}
\usepackage{booktabs}
\usepackage{tablefootnote}

\usepackage{amsmath,amsthm,amssymb,amsfonts}
\usepackage{algorithm}
\usepackage{algorithmic}
\usepackage{enumerate}
\usepackage{cleveref}
\usepackage{comment}
\usepackage{bm}
\usepackage{pifont}

\theoremstyle{plain}
\newtheorem{theorem}{Theorem}[]

\newtheorem{lemma}[]{Lemma}

\newcommand{\vertiii}[1]{{\left\vert\kern-0.25ex\left\vert\kern-0.25ex\left\vert #1 
\right\vert\kern-0.25ex\right\vert\kern-0.25ex\right\vert}}

\newtheorem{definition}{Definition}[]

\newtheorem{remark}{Remark}[]

\DeclareMathOperator\supp{supp}
\def\x{{\mathbf x}}
\def\w{{\boldsymbol{\omega}}}

\def\e{{\boldsymbol{\nu}}}
\def\a{{\mathbf a}}

\def\e{{\mathbf e}}
\def\y{{\mathbf y}}
\def\v{{\mathbf v}}

\def\A{{\mathbf A}}
\def\c{{\mathbf c}}
\def\I{{\mathbf I}}

\def\P{{\mathbb{P}}}

\def\E{{\mathbb E}}
\def\O{{\mathcal O}}
\def\R{{\mathbb{R}}}
\def\B{{\mathbb{B}}}
\def\C{{\mathbb C}}
\def\S{{\mathcal S}}

\def\D{{\mathcal D}}
\def\X{{\mathcal X}}
\def\F{{\mathcal F}}
\def\N{{\mathcal N}}

\def\d{\textnormal{d}}

\newcommand{\ts}{\textsuperscript}

\usepackage{cite}
\usepackage{import}
\usepackage{amsmath}
\usepackage{mathtools}
\usepackage{dsfont}
\usepackage[utf8]{inputenc} 
\usepackage{csquotes}

\begin{document}
\title{\bf Generalization Bounds for \\ Sparse Random  Feature Expansions}
\date{}

\author[1]{Abolfazl Hashemi}
\author[2]{Hayden Schaeffer} 
\author[1]{Robert Shi}
\author[1]{\\Ufuk Topcu}
\author[3]{Giang Tran}
\author[1]{Rachel Ward\thanks{Authors are listed in  alphabetical order. This work was supported in part by AFOSR MURI FA9550-19-1-0005, AFOSR MURI FA9550-21-1-0084, NSERC Discovery Grant RGPIN 2018-06135, NSF DMS-1752116, and NSF DMS-1952735. The code is available on our github page \url{https://github.com/GiangTTran/SparseRandomFeatures}
}}
\affil[1]{The University of Texas at  Austin}
\affil[2]{Carnegie Mellon University}
\affil[3]{University of Waterloo}
\maketitle

\begin{abstract}
Random feature methods have been successful in various machine learning tasks, are easy to compute, and come with theoretical accuracy bounds.  They serve as an alternative approach to standard neural networks since they can represent similar function spaces without a costly training phase. However, for accuracy,  random feature methods require more measurements than trainable parameters, limiting their use for data-scarce applications or problems in scientific machine learning. This paper introduces the sparse random feature expansion to obtain parsimonious random feature models.  Specifically, we leverage ideas from compressive sensing to generate random feature expansions with theoretical guarantees even in the data-scarce setting. In particular, we provide generalization bounds for functions in a certain class (that is dense in a reproducing kernel Hilbert space) depending on the number of samples and the distribution of features. The generalization bounds improve with additional structural conditions, such as coordinate sparsity, compact clusters of the spectrum, or rapid spectral decay. In particular, by introducing sparse features, i.e. features with random sparse weights, we provide improved bounds for low order functions. We show that the sparse random feature expansions outperforms shallow networks in several scientific machine learning tasks.
\end{abstract}

\section{Introduction}\label{sec:intro}
The \textit{sparsity-of-effects} or \textit{Pareto principle} states that most real-world systems are dominated by a small number of low-complexity interactions. This idea is at the heart of compressive sensing and sparse optimization, which computes a sparse representation for a given dataset using a large set of features. The feature spaces are often constructed using a random matrix, e.g., each element is independent and identically distributed from the normal distribution, or constructed using a bounded orthonormal system, e.g., Fourier or orthonormal polynomials. While completely random matrices are useful for compression, their lack of structure can limit applications to problems that require physical or meaningful constraints. On the other hand, while bounded orthonormal systems provide meaningful structure to the feature space, they often require knowledge of the sampling measure and the target functions themselves, e.g., that the target function is well-represented by polynomials.

In the high-dimensional setting, neural networks can achieve high test accuracy when there are reasonable models for the local interactions between variables. For example, a convolutional neural network imposes local spatial dependencies between pixels or nodes. In addition, neural networks can construct data-driven feature spaces that far exceed the limitations of pre-specified bases such as polynomials. However, standard neural networks often rely on back-propagation or greedy algorithms to train the weights, which is a computationally intensive procedure. Furthermore, the trained models do not provide interpretable results, i.e., they remain black-boxes. Randomized networks are a class of neural networks that randomize and fix the weights within the architecture \cite{block1962perceptron, rahimi2007random, rahimi2008uniform, maass2004computational, moosmann2006randomized}. When only the final layer is trained, the training problem becomes linear and can have a much lower cost than the non-convex optimization-based approaches. This method has motivated new algorithms and theory, for example, see \cite{rahimi2007random, rahimi2008uniform, rahimi2008weighted, sutherland2015error, yang2012nystrom, sriperumbudur2015optimal, chitta2012efficient, li2019towards}. Recently, generalization bounds for over-parameterized random features ridge regression were provided in \cite{mei2020generalization},  when the Tikhonov regularization parameter tends to zero. The analysis is asymptotic and is restricted to the ReLU activation function, with data and features drawn on the sphere.

In this work, we introduce a new framework for approximating high-dimensional functions in the case where measurements are \emph{expensive} and \emph{scarce}. We propose the  \textit{sparse random feature expansion} (SRFE), which enhances the compressive sensing approach by allowing for more flexible functional relationships between inputs, as well as a more complex feature space. The choice of basis is inspired by the random Fourier feature (RFF) method \cite{rahimi2007random, rahimi2008uniform}, which uses a basis comprised of simple (often trigonometric) functions with randomized parameters. In the RFF method, the model is learned using ridge regression, which leads to dense (or full) representations.  By using sparsity, our approach could be viewed as a way to leverage structure in the data-scarce setting while retaining the accuracy and representation capabilities of the randomized feature methods. In addition, the use of sparsity allows for  reasonable generalization bounds even in the very overcomplete setting, which is proving to be a powerful modern tool related to over-parameterized neural networks \cite{jacot2018neural, du2018gradient, li2018learning, arora2019fine}. 

In terms of the approximation error, the randomized methods can achieve similar results to those associated with shallow networks. In \cite{jones1992simple, barron1993universal}, it was shown that if the Fourier transform of the target function $f$, denoted by $\hat{f}$, has finite integral $\int_{\mathbb{R}^d} |\w| |\hat{f}(\w)| d\w$ then there is a two-layer neural network with $N$ terms that can approximate $f$ up to an $L^2$ error of $\O(N^{-\frac{1}{2}})$.   These results (and their generalizations) often require specific (greedy) algorithms to achieve. In addition, neural networks often only achieve good performance in the data-rich and over-parameterized regimes. 
On the other hand,  the RFF method achieves uniform errors on the order of $\O(N^{-\frac{1}{2}})$ for functions in a certain class (associated with the choice of the basis functions) without the need for a particular algorithm or construction \cite{rahimi2008uniform}. Generalization error bounds for random feature ridge regression from \cite{sriperumbudur2015optimal, rudi2017generalization, ullah2018, szabo2019kernel} also achieve the rate   $\O(N^{-\frac{1}{2}})$, provided the number of data samples grows with $N$ and satisfies certain statistical assumptions.  Our generalization bounds for random feature expansions obtained by $\ell_1$-minimization match this rate in the general setting without needing a rich training set.  Specifically, we show that if the underlying function is a low-order function, admitting a decomposition into a small number of functions each of which depends on only a few variables, then \emph{sparse} random feature expansions can achieve generalization bound $\O(N^{-\frac{1}{2}})$ with constants that depend on a polynomial (and not an exponential) of the dimension, in this sense, overcoming the curse of dimensionality.

One of the most popular techniques in the area of uncertainty quantification is the Polynomial Chaos Expansion (PCE). PCE models are built up from univariate orthonormal polynomial regression; in particular, each basis term is the product of univariate orthonormal polynomials and is characterized by the multi-index of polynomial degrees in each direction. The standard PCE approach solves for the coefficients of the polynomials using the ordinary least squares method. The sparse PCE has recently gained traction, where the coefficient vector is determined through sparse regression.  Many sparse regression methods used in PCE were originally developed for compressive sensing \cite{donoho2006compressed, candes2006stable, rauhut2012sparse, doostan2011non}. The success of sparse PCE is due in part to the method's ability to incorporate higher degree terms without overfitting. However, the polynomial basis must be orthogonalized with respect to the sampling measure.  Moreover, good performance is limited to functions which are well-represented by moderate degree polynomials. This serves as another motivation for the use of randomized features, which may increase the richness of the approximation.

\subsection{Contribution} 
\label{sec:contribution}

We propose a sparse feature model (the SRFE) which improves on compressive sensing and PCE approaches by utilizing random features from the RFF model. Also, the SRFE outperforms a standard shallow neural network in the limited data regime.  We incorporate sparsity in the proposed model in two ways. The first is in our approximation of the target function by using a small number of terms from a large feature space to represent the dominate behavior (this is the sparse expansion component). The second level of sparsity can be considered as side information on the variables and is incorporated by sampling random low order interactions between variables (the sparse features).  Building upon these ideas, as part of our theoretical contributions, we derive sample and feature complexity bounds such that the error between the SRFE and the target function is controlled by the richness of the random features, the compressibility of the representation, and the noise on the samples (formalized in \Cref{sec:functions}). This also shows the tractability of sparse expansions in the context of randomized feature models.

The SRFE offers additional freedom through redundancy of the basis and does not restrict the model class to low order interactions in the form of polynomials.  While our main results are stated for trigonometric features, extensions and applications with ReLU and other standard activation functions can be derived in the same way. In addition, our method and analysis could be extended to include different sampling strategies such as those used in the recovery of dynamical systems
\cite{schaeffer2018extracting, schaeffer2020extracting}.

In order to provide generalization bounds, we first characterize the approximation power of the best fit approximator; then, we bound the error between the best fit and the sparse random feature expansion. The best fit results are extensions of \cite{rahimi2007random, rahimi2008uniform}, but we provide the proof for completeness. The generalization bounds and the sparse approximation results are both novel. While we utilize standard coherence-based results for sparse recovery, we prove new bounds for the coherence and the sample complexity based on the randomized features (for both dense and sparse features). It is important to note that the bounds are meaningful even when the sparsity increases, which deviates from the standard compressing sensing results. In \cite{yen2014sparse}, a sparse random feature algorithm is proposed which iteratively adds random features by using a combination of LASSO and hard thresholding. In our work, we provide sample complexity, sparsity guarantees, and generalization bounds which did not appear in previous works. In addition, we introduce sparse feature weights within our model, which can help with the curse-of-dimensionality for approximating low order functions. 

The works of \cite{rosset2007,du2020few,rakotomamonjy2013learning} consider the problem of multi-task learning to learn prediction functions,  for $T$ known tasks that lead to the lowest regularized empirical risk.  This differs from our algorithmic and theoretical contributions, which focus on the setting of approximating high dimensional low order functions with unknown interactions.
In \cite{shi2007detecting,guo2012bayesian,lim2015learning}, the aim is to learn pairwise interactions with linear regression or logistic regression using sparsity-promoting approaches such as LASSO and group-LASSO. As a comparison, we provide generalization bounds which did not appear in previous works.
It is also worth noting that our method extends to any algorithm that uses coherence-based sparsity guarantees, for example, greedy methods such as orthogonal matching pursuit, and the alternative formulation of the basis pursuit or LASSO problem in \cite{figueiredo2007gradient}.

A related direction is that of sparse learning-based additive models for kernel regression \cite{ravikumar2009sparse,kandasamy2016additive,chen2017group,liu2020sparse}. In \cite{kandasamy2016additive}, the authors propose the shrunk additive least square approximation (SALSA) method to utilize the interactions among the variables/features, which in some sense, is related to our aim in this paper to leverage the low-order interaction. However, our approach differs from SALSA since we consider sparse feature selection. Furthermore, \cite{kandasamy2016additive} establish bounds on the expected generalization error while we provide high-probability generalization bounds for our proposed method. Recently, \cite{liu2020sparse} considered SALSA with an $\ell_1$ penalty and established high-probability generalization bounds; however, it is limited to kernel regression with exact kernels. In contrast with \cite{liu2020sparse}, we provide explicit sparsity guarantees along with the generalization bounds. Moreover, while \cite{kandasamy2016additive,liu2020sparse} focus on exact kernels, we leverage random features \cite{rahimi2007random,rahimi2008uniform,rahimi2008weighted} for efficient function approximation.

\section{Approximation via Sparse Random Feature Expansion (SRFE)}\label{sec:alg}

\textbf{Notation.} Throughout this paper, we use bold letters and bold capital letters to denote column vectors and matrices, respectively (e.g., $\x$ and $\A$). Let $[N] = \{1,\dots,N\}$ for any positive integer $N$ and $\|\c\|$ denote the Euclidean norm of a vector $\c$.  Throughout the paper, $f$ denotes functions of $d$ variables while $g$ denotes functions of $q\ll d$ variables. Furthermore, $\B^d(M)$ denotes the Euclidean ball in $\mathbb{R}^d$ of radius $M$.  
A vector $\c \in \C^N$ is said to be $s$-sparse if the number of nonzero components of $\c$ is at most $s$.
For a vector $\c\in \C^N$, let $\kappa_{s,p}(\c)$ denote the error of best $s$-term approximation to $\c$ in the $\ell_p$ sense, 
$\kappa_{s,p}(\c) := \min\{ \| \c - {\bf z} \|_{\ell_p}: {\bf z} \text{ is $s$-sparse} \}$ \cite{foucartmathematical}. Note in particular that $\kappa_{s,p}(\c) = 0$ if $\c$ is $s$-sparse, and 
$\kappa_{s,p}(\c) \leq \| \c \|_{\ell_p}$ always. 

We are interested in identifying an unknown function $f: \R^d \rightarrow \C$, belonging to a certain class (defined in \Cref{sec:functions}), from a set of samples. We assume that the  $m$ sampling points $\x_k$'s are drawn with a probability measure $\mu(\x)$ with the corresponding output values 
\begin{equation}
y_k = f(\x_k)+e_k, \quad |e_k|\leq E, \quad \forall k \in [m],
\end{equation}
where $e_k$ is the noise.

A fundamental approach in approximation theory relies on the assumption that $f$ has an approximate linear representation with respect to a suitable collection of $N$ functions $\phi_j(\x)$, $j\in[N]$:
\begin{equation}\label{eq:linear_basis_approximation}
f(\x) \approx \sum_{j=1}^N c_j\phi_j(\x).
\end{equation}
Important examples of such families of functions include real and complex trigonometric polynomials as well as Legendre polynomials \cite{rauhut2012sparse,rauhut2016interpolation,adcock2017compressed, adcock2018infinite,chkifa2018polynomial}.

Let $\A \in \mathbb{C}^{m \times N}$ be the random feature matrix with entries
$a_{k,j} = \phi_j(\x_k)$, then approximating $f$ in \Cref{eq:linear_basis_approximation} is equivalent to
\begin{equation}
\text{find } \quad \c \in \C^N \quad \text{such that }\quad {\bf y} \approx \A {\bf c},
\end{equation}
where $\c=[c_1, \ldots, c_N]^T$ and $\y=[y_1, \ldots, y_m]^T$.
In many applications, it is often the case that $f$ is well-approximated by a small subset of the $N$ functions, which implies that $\c$ is sparse. 
By exploiting the sparsity, the number of samples $m$ required to obtain an accurate approximation of $f$ may be significantly reduced. One effective approach to learn a sparse vector $\c$ is to solve the basis pursuit (BP) problem:
\begin{equation}\label{eq:BP}
\c^\sharp = \arg\min_{\c}\quad \|\c\|_1 \qquad \text{s.t. }\quad \|\A\c-{\bf y}\|\leq \eta\sqrt{m},
\end{equation}
where $\eta$ is a parameter typically related to the measurement noise. The conditions for stable recovery of any sparse vector $\c^\star$ satisfying ${\bf y} \approx \A {\bf c}^\star$ is extensively studied in compressed sensing and statistics \cite{candes2006near,cai2009recovery,foucartmathematical}.

In order to construct a sufficiently rich family of functions, we use a randomized approach. Specifically, consider a collection of functions $\phi(\x ;\w) = \phi( \langle \x, \w \rangle)$ parameterized by a weight vector $\w$ drawn randomly from a probability distribution $\rho(\w)$. Some popular choices for $\phi$ are 
\begin{enumerate}
\item \textit{Random Fourier features:} $\phi(\x;\w) = \exp(i \langle \x,\w \rangle)$.
\item \textit{Random trigonometric features:}\\ $\phi(\x;\w)= \cos( \langle \x,\w \rangle)$ and  $\phi(\x;\w)= \sin( \langle \x,\w \rangle)$.
\item \textit{Random ReLU features:} $\phi(\x;\w) = \max(\langle \x,\w \rangle, 0)$.
\end{enumerate}
Based on \cite{rahimi2007random, rahimi2008uniform}, we call such $\phi(\cdot\, ;\w)$ the random features. 
Altogether, we propose the \textit{Sparse Random Feature Expansion (SRFE)} to approximate $f$, which is summarized in \Cref{alg:B}.

\begin{algorithm}[H]
\caption{Sparse Random Feature Expansion (SRFE)}
\label{alg:B}
\begin{algorithmic}[1]
\STATE {\bfseries Input:} parametric basis function $\phi( \hspace{.5mm} ;\w) = \phi( \langle \x, \w \rangle)$, stability parameter $\eta$.
\STATE Draw $m$ data points $\x_k \sim \D_x$ and observe outputs $y_k = f(\x_k)+e_k$ with $|e_k|\leq E$.
\STATE Draw $N$ random weights $\w_j\sim \D_\omega$ (independent of the $\x_k$'s).
\STATE Construct the random feature matrix $\A \in \C^{m \times N}$ such that $a_{kj} = \phi(\x_k; \w_j)$.
\STATE Solve \vspace{-3mm}
\begin{equation}\nonumber
\c^\sharp = \arg\min_{\c}\quad \|\c\|_1 \qquad \text{s.t. }\quad \|\A\c-{\bf y}\|\leq \eta\sqrt{m}.
\end{equation}
\STATE (Optional) Pruning: Set $\S^\sharp$ to be the support set of the $s$ largest (in magnitude) coefficients of $\c^\sharp$ and redefine $\c^\sharp$ to be zero outside of $\S^\sharp$.
\STATE {\bfseries Output:} Form the approximation 
\vspace{-3mm}
\begin{equation}\nonumber
f^\sharp(\x) = \sum_{j=1}^N \c^\sharp_j\, \phi(\x;\w_j).
\end{equation}
\vspace{-3mm}
\end{algorithmic}
\end{algorithm}

\section{Low Order Functions}\label{sec:functions}
Often, high dimensional functions that arise from important physical systems are of low order, meaning the function is dominated by a few terms, each depending on only a subset of the input variables, say $q$ out of the $d$ variables where $q \ll d$ \cite{kuo2010decompositions,devore2011approximation}.
Low order functions also appear in other applications as a way to reduce modeling complexity. For example, in dimension reduction and surrogate modeling, sensitivity analysis is employed to determine the most influential input variables and thus to reduce the approximation onto a subset of the input space \cite{saltelli2008global}. The notion of low order functions are also connected to low-dimensional structures \cite{potts2019approximation, potts2019learning} and active subspaces \cite{fornasier2012learning, constantine2014active,constantine2017near}. Low order additive functions and sparsely connected networks are also well-motivated in computational neuroscience for simple brain architectures \cite{harris2019additive}.

Next, we formalize the notion of low order functions by extending the definition from \cite{kuo2010decompositions}.
\begin{definition}[\textbf{Order-$q$ Functions}]
\label{def:function_class3} Fix $d,q,K\in\mathbb{N}$ with $q\leq d$. A function $f: \mathbb{R}^d \rightarrow \mathbb{C}$ is an order-$q$ function of at most $K$ terms if there exist $K$ functions $g_1, \dots, g_K: \mathbb{R}^q \rightarrow \mathbb{C}$ such that 
\begin{align}\label{eq:order-q-def}
f(x_1,\dots, x_d) &= \frac{1}{K}\sum_{j=1}^K g_j(x_{j_1}, \dots, x_{j_q})= \frac{1}{K}\sum_{j=1}^K g_j(\x|_{\S_j}),
\end{align}
where $\S_j = \{j_1,\dots,j_q\}$ is a subset of the index set $[d]$, $\S_j\cap \S_{j'} =\varnothing$ for $j\not= j'$, and $\x|_{\S_j}$ is the restriction of $\x$ onto $\S_j$. 
\end{definition}

Note that in general, such a  decomposition is not unique. Furthermore, we are interested in the smallest $q$ to refer to the order of a function; trivially, any order-$q$ function $f: \mathbb{R}^d \rightarrow \mathbb{C}$ is also order-$d$. 

With this side information, we can further reduce the number of samples needed (see Theorem \ref{thm:low_order}). We modify \Cref{alg:B} to incorporate the potential coordinate sparsity into the weights $\w$. Since we do not know the set of active variables, we draw a number of sparse random feature weights on every subset $\mathcal{S} \subset [d]$ of size $|\mathcal{S}| = q$.  That is, for each such $\mathcal{S}$, we draw the on-support feature components randomly from the given distribution, and we set the remaining components to be zero. In particular, we have the following definition for our random features.

\begin{definition}[\textbf{$q$-Sparse Feature Weights}]\label{def:sparse_features}
Let $d,q,n\in\mathbb{N}$ with $q\leq d$ and a multivariate probability density $\zeta: \R^q \rightarrow \R$. A collection of $N = n{\genfrac(){0pt}{2}{d}{q}}$ weight vectors $\w_1, \dots, \w_N$ is said to be a \emph{complete set of $q$-sparse feature weights (drawn from density $\zeta$)} if they are generated as follows: For each subset $\mathcal{S}\subset  [d]$ of size $|\mathcal{S}| = q $, draw $n$  random vectors $z_1, \dots, z_n \in \mathbb{R}^q$ from $\zeta$, independent of each other and of all previous draws. Then, use $z_1, \dots, z_n$ to form $q$-sparse feature weights $\w_{1}, \dots, \w_{n} \in \mathbb{R}^d$ by setting  $\supp(\w_{k}) = \mathcal{S}$ and $\w_{k} \big|_{{\mathcal{S}}} = z_{k}$. 
\end{definition}

This leads to the Sparse Random Feature Expansion with Sparse Features (\textit{SRFE-S}) by modifying Step (3) of \Cref{alg:B}  to ``Draw a complete set of $N$ $q$-sparse feature weights $\w_j\in \R^d$ sampled from density $\zeta: \mathbb{R}^q \rightarrow \mathbb{R}$''. We summarize SRFE-S in \Cref{alg:B2}.

\begin{algorithm}[h]
\caption{Sparse Random Feature Expansion with Sparse Feature Weights (SRFE-S)}
\label{alg:B2}
\begin{algorithmic}[1]
\STATE {\bfseries Input:} parametric basis function $\phi(\x;\w) = \phi( \langle \x, \w \rangle)$, feature sparsity level $q$, probability density $\zeta: \mathbb{R}^q \rightarrow \mathbb{R}$, stability parameter $\eta$.
\STATE Draw $m$ data points $\x_k \sim \D_x$ and observe outputs $y_k = f(\x_k)+e_k$ with $|e_k|\leq E$.
\STATE Draw a complete set of $N$ $q$-sparse feature weights $\w_j\in \R^d$ sampled from density $\zeta: \mathbb{R}^q \rightarrow \mathbb{R}$ as defined in \Cref{def:sparse_features} (and independent of the $\x_k$'s).
\STATE Construct a random feature matrix $\A \in \C^{m \times N}$ such that $a_{kj} = \phi(\x_k;\w_j)$.
\STATE Solve \vspace{-3mm}
\begin{equation}\nonumber
\c^\sharp = \arg\min_{\c}\quad \|\c\|_1 \qquad \text{s.t. }\quad \|\A\c-{\bf y}\|\leq \eta\sqrt{m}.
\end{equation}
\STATE (Optional) Pruning: Set $\S^\sharp$ to be the support set of the $s$ largest (in magnitude) coefficients of $\c^\sharp$ and redefine $\c^\sharp$ to be zero outside of $\S^\sharp$.
\STATE {\bfseries Output:} Form the approximation 
\vspace{-3mm}
\begin{equation}\nonumber
f^\sharp(\x) = \sum_{j=1}^N \c^\sharp_j\,\phi(\x;\w_j).
\end{equation}
\end{algorithmic}
\end{algorithm}

\begin{remark}
 Drawing a complete set of $q$-sparse feature weights can be slow and cumbersome.   In the case where ${\bf \zeta}(x_1, \dots, x_q) = \prod\limits_{j=1}^q \zeta(x_j)$ is a tensor product of univariate densities, a significantly more practical method for drawing sparse features is as follows: we randomly generate a size $q$ subset of $[d]$ and then define the on-support values using $\zeta$. Alternatively, one can draw sparse feature weights by the following procedure: for every $k\in [N]$, the $j$-th entry of $\w_k\in\R^d$, $\w_{k,j},$ is set to 0 with probability $\left(1-\dfrac{q}{d}\right)$ and is drawn from $\zeta$, $\w_{k,j} \sim \zeta,$ with probability $\dfrac{q}{d}$. We further note that any side-information on the feasibility of the low order support subsets can be incorporated in the procedure outlined in \Cref{alg:B2} to further reduce the required number of sparse features.
\end{remark}

\section{Theoretical Analysis}\label{sec:thm}
In this section, we provide theoretical performance guarantees on the approximation given by \Cref{alg:B} and \Cref{alg:B2}.  In particular, we derive an explicit bound on the required number of data samples for a stable approximation within a target region. Given the connections to Fourier analysis and its desired  characteristics, we mainly focus on the case where $\phi(\x;\w) = \exp( i \langle \x,\w \rangle)$. Nonetheless, these results extend to other distributions and basis functions.

Before stating the main results, we recall some useful definitions. The first definition is a complex-valued extension of the class introduced in \cite{rahimi2008uniform}.  
\begin{definition}[\textbf{Bounded $\rho$-norm Functions}]\label{def:function_class1}
Fix a probability density function $\rho: \mathbb{R}^d \rightarrow \mathbb{R}$ and a function $\phi:  \mathbb{R}^d\times \mathbb{R}^d \rightarrow \mathbb{C}$.  A function $f: \mathbb{R}^d \rightarrow \mathbb{C}$ has finite $\rho$-norm with respect to $\phi(  \x; \w )$ if it belongs to the class
\begin{multline}
\mathcal{F}({\phi,\rho}) := \Bigg\{f(\x) = \int_{\w \in \mathbb{R}^d} \alpha(\w) \phi(  \x; \w )\;d\w \; : \; \|f\|_\rho := \sup_\w \left|\frac{ \alpha(\w)}{\rho(\w)}\right|<\infty\Bigg\} .
\end{multline}
\end{definition}

Note that in the above definition, if $\phi(  \x; \w ) = \exp( i \langle \x,\w \rangle)$, $\alpha: \mathbb{R}^d \rightarrow \mathbb{C}$ is the inverse Fourier transform of $f$.

\subsection{Generalization Error}
We state our main results here.  Recall that $\mu(\x)$ denotes the probability measure for sampling $\x$.

\begin{theorem}[\textbf{Generalization Bound for Bounded $\rho$-norm Functions}]\label{thm:generalization} Let $f \in \F(\phi,\rho)$, where $\phi(\x; \w)= \phi(\langle \x, \w \rangle) = \exp(i\langle \x, \w\rangle )$ and $\rho(\w)$ is the density corresponding to a  spherical Gaussian with variance $\sigma^2$, $ \mathcal N(\mathbf{0}, \sigma^2 \I_d).$ For a fixed $\gamma$, consider a set of data samples $\x_1,\dots, \x_m\sim \mathcal{N}(\mathbf{0}, \gamma^2 \I_d)$ and frequencies $\w_1,\dots, \w_N\sim \mathcal N(\mathbf{0}, \sigma^2 \I_d)$. The measurement noise $e_k$ is either bounded by $E=2\nu$ or to be drawn i.i.d. from $\N(0,\nu^2)$. Let $\A\in \C^{m\times N}$ denote the associated  random feature matrix where $a_{k, j}=\phi(\x_k;\w_j)$. Let $f^\sharp$ be defined from \Cref{alg:B} and \Cref{eq:BP} with $\eta=\sqrt{2(\epsilon^2\|f\|_\rho^2+E^2)}$ and with the additional pruning step
\begin{equation}\nonumber
f^\sharp(\x) := \sum_{j\in\S^\sharp} \c^\sharp_j\, \phi(\x;\w_j).
\end{equation}
where $\S^\sharp$ is the support set of the $s$ largest (in magnitude) coefficients of $\c^\sharp$. 

For a given $s$, if the feature parameters $\sigma$ and $N$, the confidence $\delta$, and the accuracy $\epsilon$  are chosen so that the following conditions hold:
\begin{enumerate}
\item $\gamma$-$\sigma$ uncertainty principle
\begin{equation}
\gamma^2\sigma^2\geq \frac{1}{2}\left(\left(\frac{\sqrt{41}(2s-1)}{2}\right)^{\frac{2}{d}}-1\right),
\end{equation}
\item Number of features
\begin{equation}
N =  \frac{4}{\epsilon^2}\left(1 + 4 \gamma\sigma  d \sqrt{1+\sqrt{\frac{12}{d} \log\frac{m}{\delta}}}+ \sqrt{\frac{1}{2}\log\left(\frac{1}{\delta}\right)}\right) ^2,
\end{equation}
\item Number of measurements
\begin{equation}
\begin{aligned}
&m \geq 4(2\gamma^2\sigma^2+1)^{d} \log \frac{N^2}{\delta}.
\end{aligned}
\end{equation}
\end{enumerate}
Then, with probability at least $1-5\delta$ the following error bound holds
\begin{equation}
\begin{aligned}
\sqrt{\int_{\R^d} | f(\x) - f^{\sharp}(\x) |^2\, \d\mu } &\leq  C' \left(1+\, N^{\frac{1}{2}} \, s^{-\frac{1}{2}} \, m^{-\frac{1}{4}} \log^{1/4} \left(\frac{1}{\delta}\right)\right) \, \kappa_{s,1}(\c^\star)\\
&+ C  \left( 1+ N^{\frac{1}{2}}m^{-\frac{1}{4}}\, \log^{1/4} \left(\frac{1}{\delta}\right)\right) \, \sqrt{\epsilon^2 \, \|f\|_\rho^2+4\nu^2},
\end{aligned}
\end{equation}
where $C,C'>0$ are constants and $\c^\star$ is the vector
\begin{equation}
    \c^\star=\frac{1}{N}\left[\frac{\alpha(\w_1)}{\rho(\w_1)}, \cdots,\frac{\alpha(\w_N)}{\rho(\w_N)}\right]^T.
\end{equation}
\end{theorem}

\begin{remark}
Although the bounds include a factor of $N^{\frac{1}{2}}$, the error decreases with $N$ in many settings. For example, let's consider the noise-free case $E=0$ and set $s=N$ i.e. the upper bound for the sparsity. The first term becomes zero and the remaining term simplifies to
 \begin{equation}
\begin{aligned}
&\sqrt{\int_{\R^d} | f(\x) - f^\#(\x) |^2 \, \d\mu } \\
&\leq   C  \left( N^{-\frac{1}{2}}+ m^{-\frac{1}{4}}\, \log^{1/4}\left(\frac{1}{\delta}\right)\right) \,\left(1 + 4 \gamma\sigma  d \sqrt{1+\sqrt{\frac{12}{d} \log\frac{m}{\delta}}}+ \sqrt{\frac{1}{2}\log\left(\frac{1}{\delta}\right)}\right) \, \|f\|_\rho,\\
&\leq   \tilde{C} \,  N^{-\frac{1}{2}}\, \left( 1+  \frac{\log^{1/4}\left(\frac{1}{\delta}\right)}{\log^{1/4}\left(\frac{N^2}{\delta}\right)}\right) \,\left(1 + 4 \gamma\sigma  d \sqrt{1+\sqrt{\frac{12}{d} \log\frac{m}{\delta}}}+ \sqrt{\frac{1}{2}\log\left(\frac{1}{\delta}\right)}\right) \, \|f\|_\rho.
\end{aligned}
\end{equation} 
where we used the complexity bounds on $m$: 
\[
m \geq 4(2\gamma^2\sigma^2+1)^{d} \log \frac{N^2}{\delta} \geq 41 (2s-1)^2 \log \frac{N^2}{\delta} \geq 41 N^2\log \frac{N^2}{\delta}.
\]
Therefore, up to log terms, our generalization bound is $\mathcal{O}(\gamma\sigma N^{-\frac{1}{2}})=\mathcal{O}(N^{-\frac{1}{2} + \frac{1}{d}})$. 
\end{remark}
\begin{remark}
Consider a function $f\in \mathcal{F}({\phi,\rho})$ whose Fourier transform is supported within a compact set $\Omega \subset \mathbb{R}^d$ such that $\int_\Omega \rho(\w) \, d\w=: \beta<1$. Then the vector $\c^\star$ will be sparse with high probability, as its expected sparsity scales like $s = \beta \, N$. Thus, functions with compactly clustered spectral energy are well-approximated by the SRFE method. 
\end{remark}
\begin{remark}
Interestingly, we observe  the appearance of a Heisenberg-type uncertainty principle between ``frequency-domain'' and ``space-domain'' variances, $\sigma^2$ and $\gamma^2$  in \Cref{thm:generalization} \cite{heisenberg1927}. In \Cref{thm:generalization}, the product of the variances are bounded below by an $\mathcal{O}(s^{\frac{2}{d}})$ term.
\end{remark}
\Cref{thm:generalization} shows that the generalization bound consists of several terms.  The first term depends on the quality of the best $s$-term approximation of $f$ with respect to the random feature basis.   Since $\kappa_{s,1}(\c^\star)$ is bounded by $\frac{N-s}{N}\|f\|_\rho$, the first error term is related to the complexity of the function class.  Part of the second term is controlled by the strength of the random features in representing $f$. By decreasing $\epsilon$, thereby increasing $N$, we can increase the power of our representation and thus reduce this error term.  The other component of the second term is proportional to the level of noise on the samples and, in general, cannot be reduced arbitrarily. However, in the high-noise case, the bound shows that taking larger $m$ will improve the error bounds with respect to the noise.

When more information is known about the target function, the rates and complexity bounds improve (especially with respect to the dimension). This helps mitigate issues with the approximation of functions in high-dimensions.  This results is detailed below.

\begin{theorem}[\textbf{Generalization Bounds for Order-$q$ Functions}]\label{thm:low_order}
Let $f$ be an order-$q$ function of at most $K$ terms 
 as defined in \Cref{def:function_class3}, such that each term $g_\ell$, $\ell=1,2,\ldots,K,$ belongs to $\F(\phi,\rho)$ with $\phi(\x; \w) = \phi(\langle \x, \w\rangle)=\exp(i\langle \x, \w\rangle ),$ and $\rho: \mathbb{R}^q \rightarrow \mathbb{R}$ the density for a spherical Gaussian with variance $\sigma^2$, ${\cal N}({\bf 0}, \sigma^2 {\bf I}_q)$. 
 Let $\w_1, \ldots, \w_N$ be a complete set of $q$-sparse feature weights drawn from density $\rho$. Fix $\gamma$ and draw i.i.d. sampling points $\x_1,\dots, \x_m\sim \mathcal{N}(\mathbf{0}, \gamma^2 \I_d)$. The measurement noise $e_k$ is either bounded by $E=2\nu$ or to be drawn i.i.d. from $\N(0,\nu^2)$.
Let $\A\in \C^{m\times N}$ denote the associated random feature matrix where $a_{k, j}=\phi(\langle\x_k, \w_j \rangle)$ and $f^\sharp$ be defined from \Cref{alg:B2} and \Cref{eq:BP} with the additional pruning step and with $\eta=\sqrt{2\epsilon^2 {d\choose q} \vertiii{f}^2+2 E^2}$, where $\vertiii{f} = \frac{1}{K} \sum\limits_{j=1}^K\|g_j\|_\rho$. 

For a given $s$, suppose the feature parameters $\sigma$ and $N$, the confidence $\delta$, and the accuracy $\epsilon$ with $\epsilon{d\choose q}^{\frac{1}{2}}$ sufficiently small are chosen so that the following conditions hold:
\begin{enumerate}
\item $\gamma$-$\sigma$ uncertainty principle
\begin{equation}
\gamma^2\sigma^2\geq \frac{1}{2}\left(\left(\frac{\sqrt{41}(2s-1)}{2}\right)^{\frac{2}{q}}-1\right),
\end{equation}
\item Number of features
\begin{equation}\label{eqn:Nbound}
N=n{d\choose q} =  \frac{4}{\epsilon^2}\left(1 + 4 \gamma\sigma  d \sqrt{1+\sqrt{\frac{12}{d} \log\frac{m}{\delta}}}+ \sqrt{\frac{q}{2}\log\left(\frac{d}{\delta}\right)}\right)^2,
\end{equation}
\item Number of measurements
\begin{equation}
\begin{aligned}
&m \geq 4(2\gamma^2\sigma^2+1)^{\max\{2q-d,0\}} (\gamma^2\sigma^2+1)^{\min\{2q,2d-2q\}}\log \frac{N^2}{\delta}.
\end{aligned}
\end{equation}
\end{enumerate}
Then, with probability at least $1-5\delta$ the following error bound holds
\begin{equation}
\begin{aligned}
&\sqrt{\int_{\R^d} | f(\x) - f^{\sharp}(\x) |^2\, \d\mu } \leq  C' \left(1+\, N^{\frac{1}{2}} \, s^{-\frac{1}{2}} \, m^{-\frac{1}{4}} \log^{1/4} \left(\frac{1}{\delta}\right)\right) \, \kappa_{s,1}(\tilde{\c}^\star)\\
&\hspace{8em}+ C  \left( 1+ N^{\frac{1}{2}}m^{-\frac{1}{4}}\, \log^{1/4} \left(\frac{1}{\delta}\right)\right) \, \sqrt{\epsilon^2 {d\choose q} \vertiii{f}^2+ E^2},
\end{aligned}
\end{equation}
where $C,C'>0$ are constants and the vector $\tilde{\c}^* =[\tilde{\c}_1^*,\ldots, \tilde{\c}_N^*]^T\in\C^N$ is defined as follows
\begin{equation}\label{eq:def:c_star2}
\tilde{\c}_j^{\star}:=\frac{1}{K}\sum_{\ell=1}^K \tilde{c}_{\ell,j}^{\star}, \quad \text{with} \ \ \tilde{c}_{\ell,j}^{\star} =\begin{cases}
\frac{\alpha_\ell(\w_j)}{n\, \rho(\w_j)}, &\text{if}\ \supp(\w_j)= \S_\ell\\
0, &\text{otherwise}.
\end{cases}\end{equation}
The function $\alpha_\ell(\w)$ is the transform of $g_\ell$ using \Cref{def:function_class1} and \Cref{def:function_class3}.
\end{theorem}

\begin{remark}
From the proof, the bound for $N$ is
\[N =\frac{4 }{\epsilon^2}\left(1 + 4 \gamma\sigma  d \sqrt{1+\sqrt{\frac{12}{d} \log\frac{m}{\delta}}}+ \sqrt{\frac{1}{2}\log\left(\frac{K}{\delta}\right)}\right)^2\]
and we obtain \Cref{eqn:Nbound} by noting that $K\leq {d\choose q} \leq \left(\frac{ed}{q}\right)^q$ (and redefining $\epsilon$).
\end{remark}
\begin{remark}
Note that in the bound for the number of measurements, the term $(\gamma^2\sigma^2+1)^{2}$ is in the range  $$ 2\gamma^2\sigma^2+1 \leq (\gamma^2\sigma^2+1)^{2} \leq (2\gamma^2\sigma^2+1)^{2}  
$$ and thus, if we choose the variances so that uncertainty principle holds with equality, then we see that $m$ scales between $s^4$ for $q\leq \frac{d}{2}$ and $s^2$ for $q=d$. 
\end{remark}

\begin{remark}
For low-order functions, Theorem \ref{thm:low_order} indicates a significant reduction in terms of the dimension $d$. In particular, for small $q$, the term ${d\choose q} \vertiii{f}$ (which includes a dimensional scale of ${d \choose q}^{\frac{1}{2}}$) should grow slower than the norm $\|f\|_{\rho'}$ where $\rho'$ is the probability density in the ambient space of dimension $d$ (assuming all terms exist). For a simple example, let $f$ be an order-$q$ function with $K=1$ and let $\alpha(\w)=1$ be compactly supported on the square defined by $(\w_1,\ldots,\w_q) \in [-1,1]^q$. If we applied \Cref{alg:B} in the ambient dimension $d$ with $\rho'$ defined as the uniform probability distribution over the square in dimension $d$, then $\|f\|_{\rho'}= 2^d$. Using sparse features with $\rho$ defined as the uniform probability distribution over the square in dimension $q$, we have $\vertiii{f} \leq {d \choose q}^{\frac{1}{2}} 2^q \leq d^{\frac{q}{2}} 2^q$. For small $q$ relative to $d$, we see that $\vertiii{f}$ will grow slower than $\|f\|_{\rho'}$ with respect to $d$ (in this example). 
\end{remark}

\subsection{Proof of Theorem \ref{thm:generalization}}\label{sec:thm:sketch}
In this section, we discuss our main technical arguments, which lead to \Cref{thm:generalization}.  Note that the generalization error can be written as
\begin{equation}\label{eq:error_sources}
\begin{aligned}
 \sqrt{\int_{\R^d} | f(\x) - f^{\sharp}(\x) |^2 \, \d\mu } &\leq   \sqrt{\int_{\R^d} | f(\x) - f^{\star}(\x) |^2 \, \d\mu }  +  \sqrt{\int_{\R^d} | f^\star(\x) - f^{\sharp}(\x) |^2 \, \d\mu } ,
\end{aligned}
\end{equation}
where 
\begin{equation}
f^{\star}(\x) = \sum_{j =1}^N c^{\star}_j \exp( i \langle \x, \w_j \rangle), \quad c^{\star}_j := \frac{\alpha(\w_j)}{N\rho(\w_j)}.
\label{eqn:fstar}
\end{equation}
We then aim to study these two sources of error in the following lemmata.
\subsubsection{Bounding the first error term}
We first extend an argument from \cite{rahimi2008uniform,rahimi2008weighted} to derive a bound on how well a function in $\F(\phi, \rho)$ can be approximated by SRFE and characterize the approximation power of $f^\star$, the best $\phi$-based approximation to $f$.
\begin{lemma}[\textbf{Generalization Error, Term 1}]\label{lem:error_term1}
Fix the confidence parameter $\delta > 0$ and accuracy parameter $\epsilon > 0$.  Recall the setting of \Cref{alg:B} and suppose $f \in \F(\phi,\rho)$ where $\phi(\x;\w) = \exp(i \langle \x,\w \rangle)$. The data samples $\x_k$ have probability measure  $\mu(\x)$ and weights $\w_j$ are sampled using the probability density $\rho(\w)$.
Consider the random feature approximation
\begin{equation}
f^{\star}(\x) := \sum_{j =1}^N c^{\star}_j \, \exp({i \langle \x, \w_j \rangle}), \quad \text{where} \ \ c^{\star}_j := \frac{\alpha(\w_j)}{N\rho(\w_j)}.
\end{equation}
If the number of features $N$ satisfies the bound
\begin{equation}
N \geq \frac{1}{\epsilon^{2}}\, \left(1 + \sqrt{2 \log\left(\frac{1}{\delta}\right)}\right)^2,
\end{equation}
then, with probability at least $1-\delta$ with respect to the draw of the weights ${\w_j}$ the following holds
\begin{equation}
\sqrt{\int_{\R^d} | f(\x) - f^{\star}(\x) |^2\, \d\mu }\leq \epsilon \|f\|_\rho.
\end{equation}
\end{lemma}

The proof of \Cref{lem:error_term1} is similar to the result of \cite{rahimi2008weighted}. The result in \Cref{lem:error_term1} is not constructive since  $\c^{\star}$ depends on the unknown function $\alpha(\w)$. Nonetheless, \Cref{lem:error_term1} establishes a useful bound on the first source of error in \eqref{eq:error_sources}.

\subsubsection{Bounding the second error term}
The next lemma controls the second source of error. 
\begin{lemma}[\textbf{Generalization Error, Term 2}]\label{lem:error_term2}
Let $f \in \F(\phi,\rho)$, where the basis function is $\phi(\x; \w)=\exp(i\langle \x, \w\rangle )$. For a fixed $\gamma$ and $q$, consider a set of data samples $\x_1,\dots, \x_m\sim \mathcal{N}(\mathbf{0}, \gamma^2 \I_d)$ with $\mu(\x)$ denoting the associated probability measure and weights $\w_1,\dots, \w_N$ drawn from $\mathcal N(\mathbf{0}, \sigma^2 \I_d)$. Assume that the noise is bounded by $E=2\nu$ or that the noise terms $e_j$ are drawn i.i.d. from $\N(0,\nu^2)$.
Let $\A\in \C^{m\times N}$ denote the associated  random feature matrix where $a_{k, j}=\phi(\x_k;\w_j)$. Let $f^\sharp$ be defined from \Cref{alg:B} and \Cref{eq:BP} with $\eta=\sqrt{2(\epsilon^2\|f\|_\rho^2+E^2)}$ and with the additional pruning step
\begin{equation}\nonumber
f^\sharp(\x) := \sum_{j\in\S^\sharp} \c^\sharp_j\, \phi(\x;\w_j),
\end{equation}
where $\S^\sharp$ is the support set of the $s$ largest (in magnitude) coefficients of $\c^\sharp$. Let the random feature approximation $f^\star$ be defined as
\begin{equation}
f^{\star}(\x) := \sum_{j =1}^N \c^{\star}_j \, \exp({i \langle \x, \w_j  \rangle)},
\end{equation}
where
 \begin{equation}
    \c^\star=\left[\frac{\alpha(\w_1)}{N\, \rho(\w_1)}, \cdots,\frac{\alpha(\w_N)}{N\,\rho(\w_N)}\right]^T.
\end{equation}
For a given $s$, if the feature parameters $\sigma$ and $N$, the confidence $\delta$, and the accuracy $\epsilon$  are chosen so that the following conditions hold:
\begin{align*}
&\gamma^2\sigma^2\geq \frac{1}{2}\left(\left(\frac{\sqrt{41}(2s-1)}{2}\right)^{\frac{2}{d}}-1\right),\\
&N =  \frac{4}{\epsilon^2}\left(1 + 4 \gamma\sigma  d \sqrt{1+\sqrt{\frac{12}{d} \log\frac{m}{\delta}}}+ \sqrt{\frac{1}{2}\log\left(\frac{1}{\delta}\right)}\right)\\
&m \geq 4(2\gamma^2\sigma^2+1)^{d} \log \frac{N^2}{\delta},
\end{align*}
then, with probability at least $1-4\delta$ the following error bound holds:
\begin{equation}
\begin{aligned}
\sqrt{\int_{\R^d} | f^\#(\x) - f^{\star}(\x) |^2 \, \d\mu }& \leq C' \left(1+\, N^{\frac{1}{2}} \, s^{-\frac{1}{2}} \, m^{-\frac{1}{4}} \log^{1/4} \left(\frac{1}{\delta}\right)\right) \, \kappa_{s,1}(\c^\star)\\
&+  C  \left( 1+ N^{\frac{1}{2}}m^{-\frac{1}{4}}\, \log^{1/4} \left(\frac{1}{\delta}\right)\right) \, \sqrt{\epsilon^2 \, \|f\|_\rho^2+4\nu^2}.
\end{aligned}
\end{equation}
where $C,C'>0$ are constants.
\end{lemma}
The proof of this lemma (see \ref{sec:general_proofs}) relies on demonstrating that given the assumptions on the data samples $\x_k$ and random weights $\w_j$, the corresponding random feature matrix $\A$ (see Step 4 in \Cref{alg:B}) has a small mutual coherence $\mu_\A$, which we recall below.
\begin{definition}[\textbf{Mutual Coherence \cite{foucartmathematical}}]\label{def:coherence}
Let $\A\in\C^{m\times N}$ be a matrix with columns $\a_1, \dots, \a_N$. The mutual coherence of $\A$ is defined as
\begin{equation}\label{eq:coherence}
\mu_\A = \sup_{\ell \neq j} \Bigg\{|\mu_{j\ell}|, \ \mu_{j\ell}:=\frac{ \langle \a_j,  \a_\ell \rangle}{ \| \a_j \| \| \a_{\ell} \|} \Bigg\}.
\end{equation}
\end{definition}
To establish \Cref{lem:error_term2}, we argue that a small mutual coherence $\mu_\A$ is itself a consequence of the  \textit{bounded separation} of the randomly drawn weights. That is, consider a collection of random weights $\{\w_j \}_{j=1}^N$ in $\R^d$. 
For $\gamma > 0$ and a function $\psi: \R^d \rightarrow \R$, we define the quantities
\begin{equation}
\begin{aligned}
&\Gamma_{j \ell} := \psi\left(\gamma(\w_{j} - \w_{\ell})\right), \quad
\Gamma_{min} := \min_{j \neq \ell} \Gamma_{j \ell} , \quad   \Gamma_{max} := \max_{j \neq \ell} \Gamma_{j \ell}.
\end{aligned}
\end{equation}
We can quantify its separation with respect to $\psi$ by bounding $\Gamma_{max}$ and $\Gamma_{min}$ by values depending on $N$ and other dimensional constants. In the setting of \Cref{thm:generalization} where the sampling points $\x_i$'s are i.i.d. Gaussian, the bounded separations hold for $\psi\left(\gamma(\w_{j} - \w_{\ell})\right) = \exp\left(-2\gamma^2\pi^2\| \w_{j} - \w_{\ell} \|^2\right)$. Consequently, by utilizing the fact that the weights $\w$'s are normally distributed, we show that the collection $\{\w_j \}_{j=1}^N$ has bounded separation by establishing  bounds on $\Gamma_{max}$ and $\Gamma_{min}$ depending on $N$. 

Given the bounds on $\Gamma_{max}$ and $\Gamma_{min}$, by employing the Bernstein's inequality, we then establish that $\mu_\A\leq 2 \Gamma_{max}$ with high probability, as long as $m\geq \frac{4}{\Gamma_{min}^2}\log \frac{N^2}{\delta}.$
Consequently, we utilize a result from compressive sensing regarding the stability of the BP formulation (see, e.g. \cite{foucartmathematical}) to complete the proof of \Cref{lem:error_term2}.

\section{Experimental Results}\label{sec:exp}
In the first example, we show that \Cref{alg:B} outperforms a shallow neural network on the approximation of an order-2 function:
\begin{align*}
f(x_1,\ldots,x_{10}) = \frac{1}{10} \, \sum_{\ell=1}^{9} \, \frac{\exp({- x_\ell^2})}{1+x_{\ell+1}^2}
\end{align*}
in the data-scarce regime. For \Cref{alg:B}, we set $\eta=0.01$,  $q=2$ or $q=10$, $\sigma = 1$, and the bias $p\sim \mathcal{U}[0,2\pi]$. In all of the examples we set $\phi(\x;\w) = \sin(\langle \x,\w \rangle)$, unless otherwise specified. We define the relative testing error to be:
\begin{align*}
\mathrm{Error}  = \sqrt{\frac{\sum_{k \in \mathrm{Test}} |f(\x_k)-f^\sharp(\x_k) |^2}{\sum_{k \in \mathrm{Test}} |f(\x_k) |^2}},
\end{align*}
where $f^\sharp$ denotes the solution of \Cref{alg:B} or the benchmarks.

In \Cref{FigComparison}, we compare the SRFE (with $N=5000$) to a two-layer ReLU network with $500$ and $5000$ trainable parameters. The ReLU network with $500$ trainable parameters is included so as to match the number of active parameters in the SRFE. The SRFE with $q=d=10$ is more accurate than the shallow network in this data regime. When $q=2$, the error of SRFE-S is smaller than that of the SRFE results with $q=d$ and is one order of magnitude smaller than the neural network.

\begin{figure}[t]
\begin{center}
\centerline{\includegraphics[width=3in]{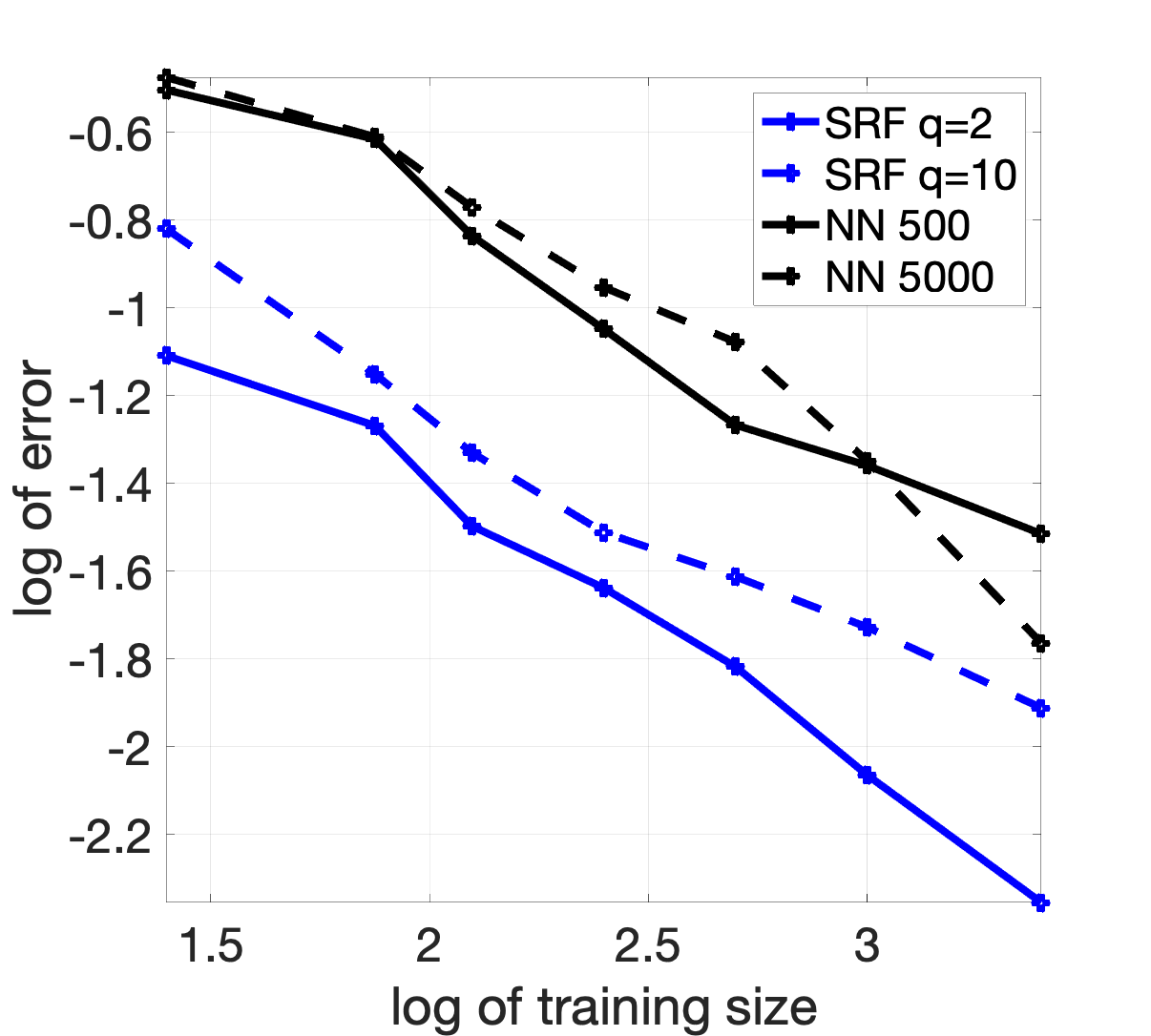}}
\caption{\small \textbf{Function Approximation:} Comparison of relative testing error versus the size of the training set for the sparse random feature model with $q=2$ and $q=10$ and for the two-layer ReLU network using 500 and 5000 trainable parameters.}
\label{FigComparison}
\end{center}
\vspace{-8mm}
\end{figure}

\subsection{Overfitting and Noise}
In this example, we provide a visual comparison of the recovery of one-dimensional functions using the SRFE algorithm and the ordinary least squares (OLS) approach. The first plot of \Cref{Fig:Comparison1d} is the target function (a sine packet), the second and third plots are the approximations using the SRFE and the OLS methods respectively with the same 200 randomly sampled points. The features are sampled using $\sigma = 2\pi$. Note the appearance of high-frequency aliasing with the OLS approximation.

\begin{figure*}[t]
\begin{center}
\includegraphics[width=1.5in]{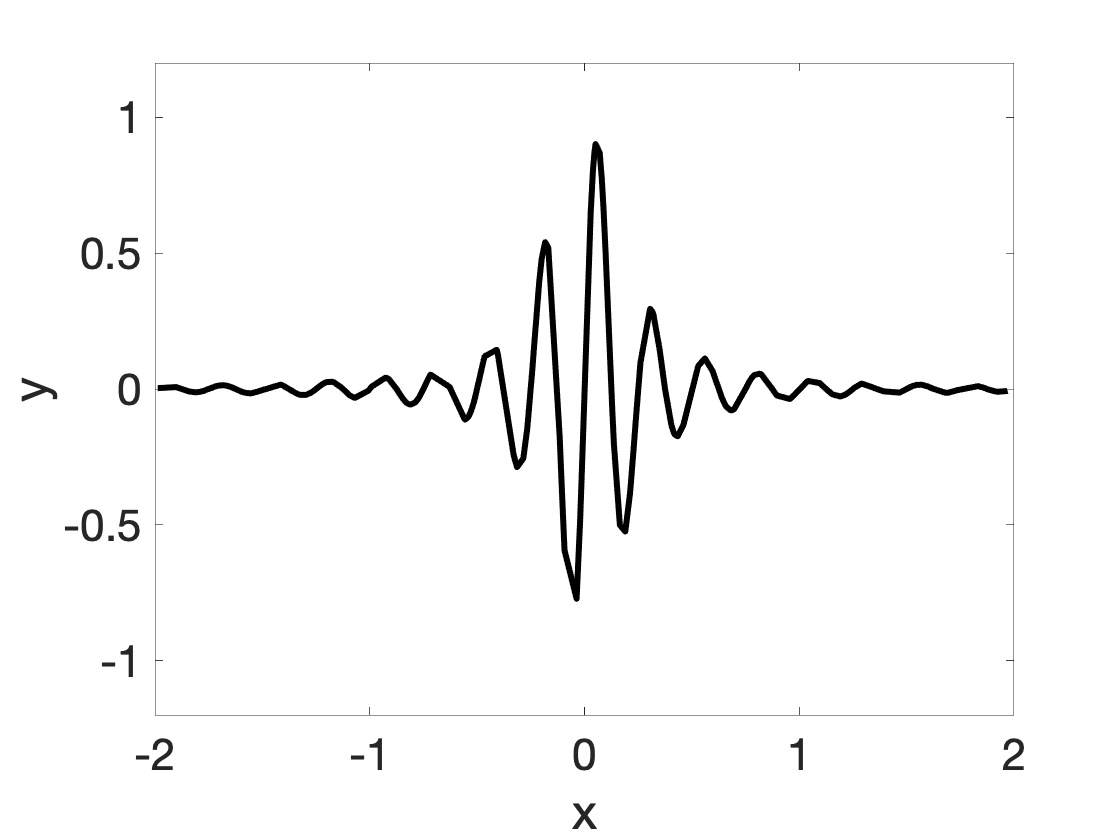} 
\includegraphics[width=1.5in]{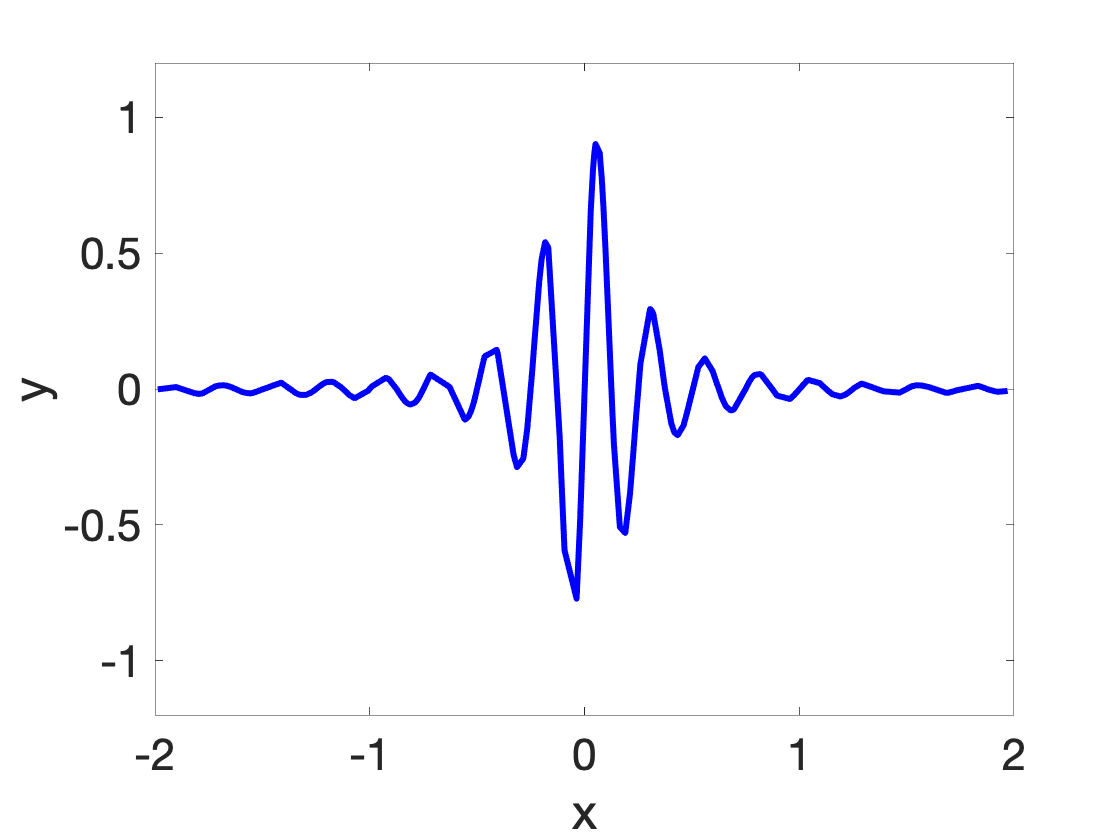} 
\includegraphics[width=1.5in]{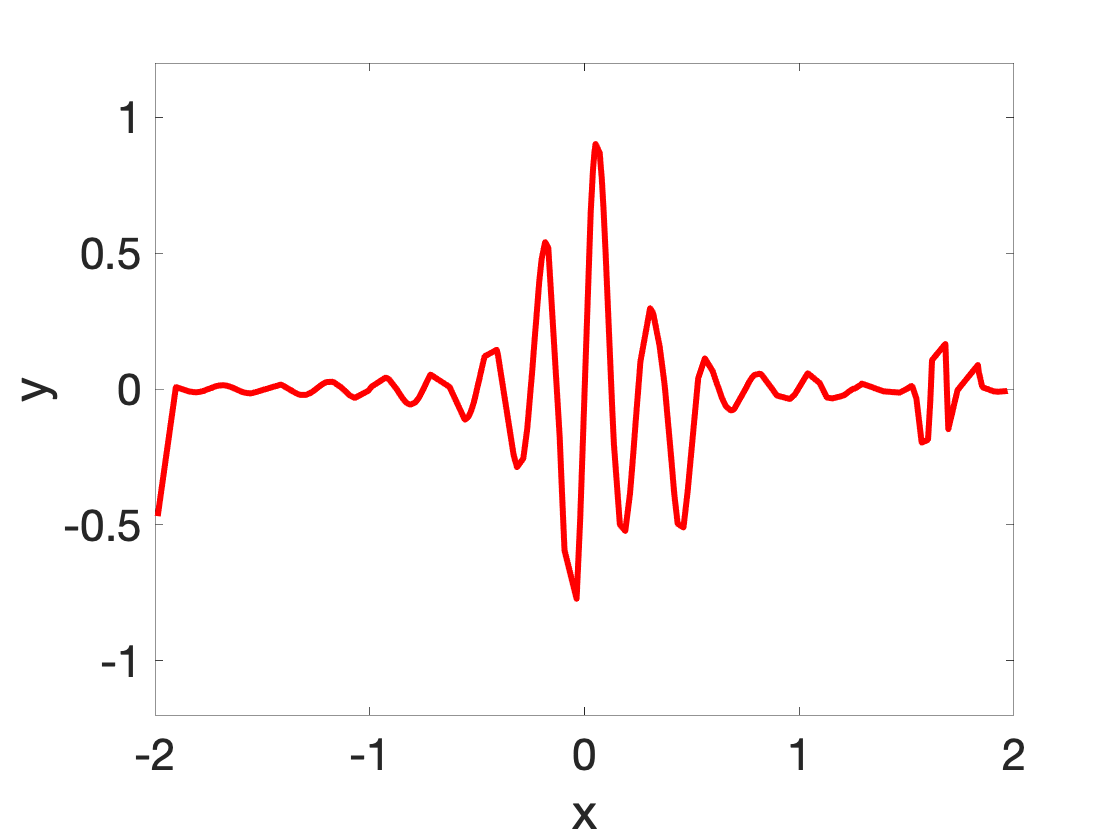} 
\caption{\small \textbf{Comparison, Overfitting}: The first figure is the target function, the second and third figures are the approximations via the SRFE and the OLS methods respectively with the same 200 randomly sampled points. 
}
\label{Fig:Comparison1d}
\end{center}
\vspace{-6mm}
\end{figure*}

In \Cref{Fig:Comparison1d2}, noisy one dimensional data is considered. The first column includes the Runge function (top) and a triangle function (bottom) each with $5\%$ relative noise. The second and third columns are the approximations using the SRFE and the OLS methods respectively with the same 200 randomly sampled points. The first row uses $\sigma = \pi$ and the second row uses $\sigma = 2\pi$. The results using the SRFE are more accurate and contain less noise artifacts. Note that since the basis is trigonometric, the approximations are smooth. The OLS results have overfit the data, even when the feature parameter $N$ is varied.

\begin{figure*}[t]
\begin{center}
\includegraphics[width=1.5in]{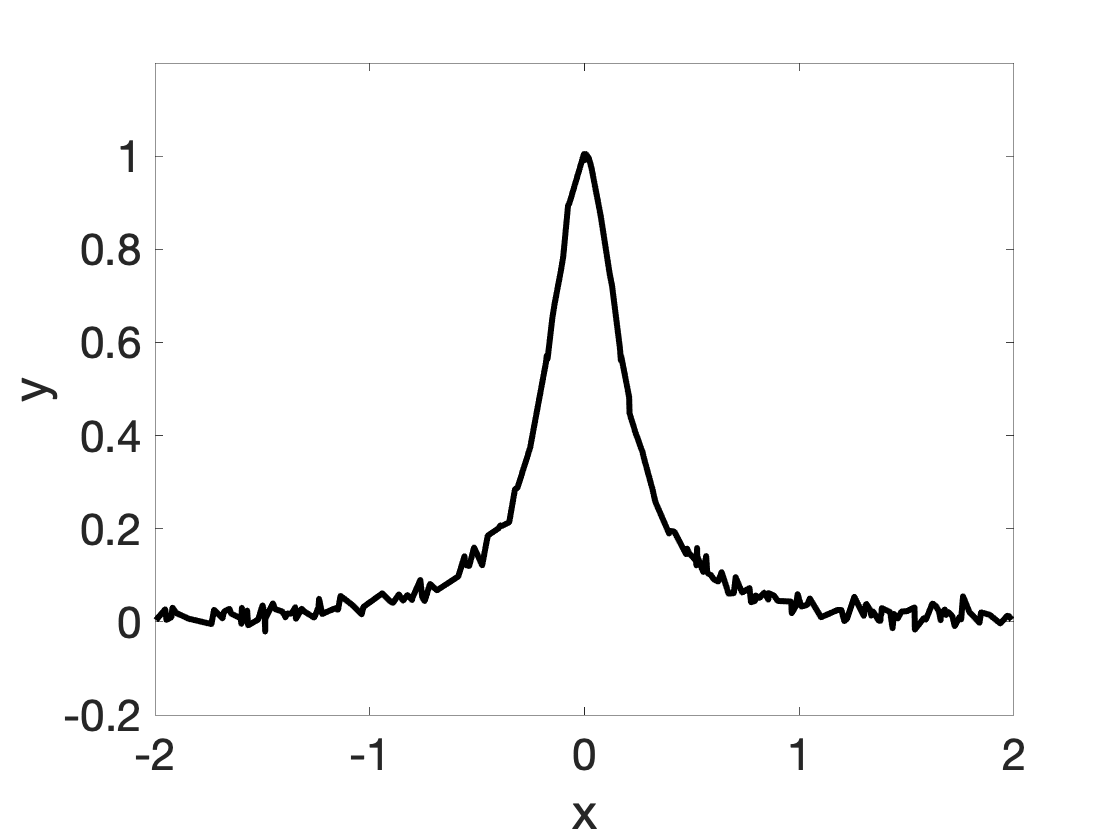} 
\includegraphics[width=1.5in]{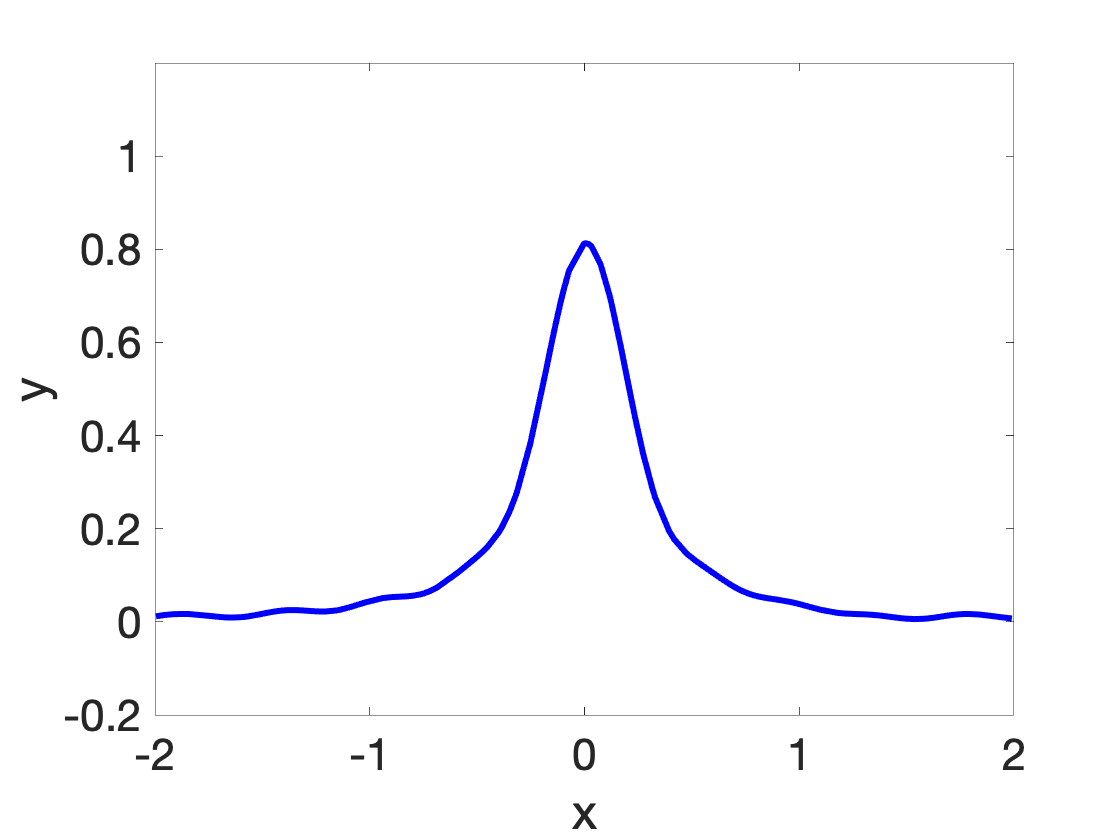} 
\includegraphics[width=1.5in]{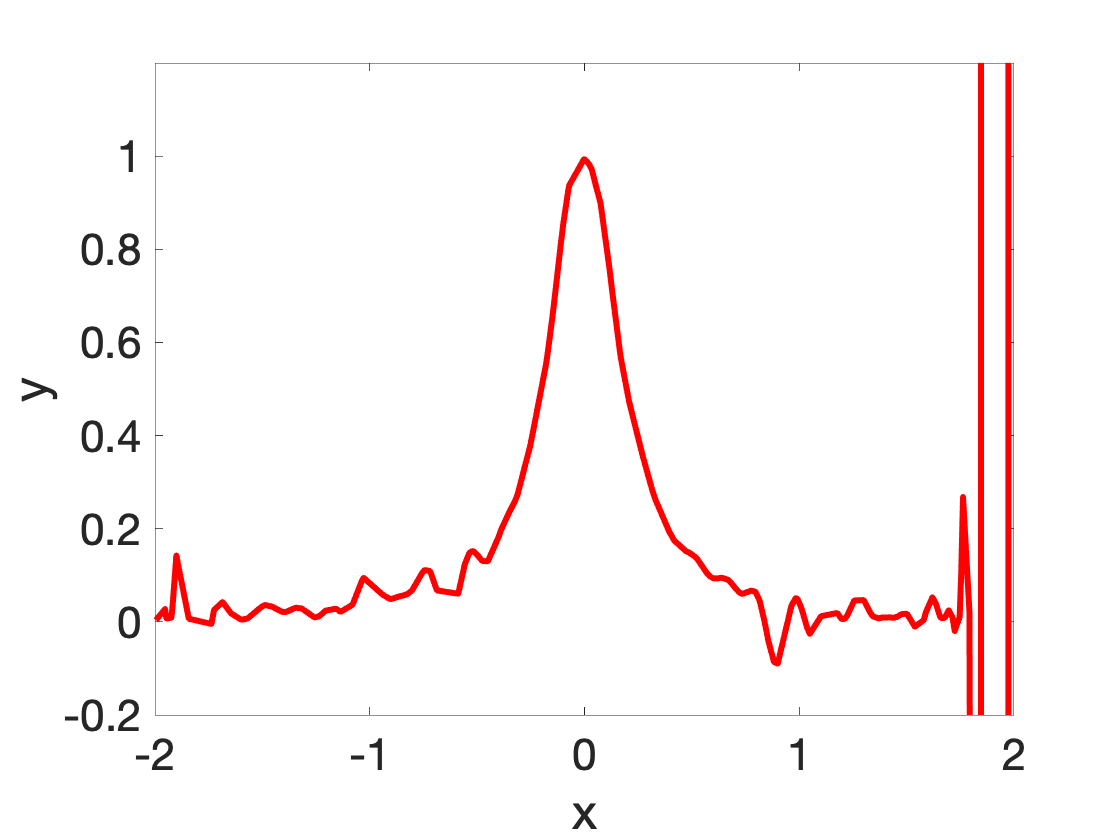} \\
\includegraphics[width=1.5in]{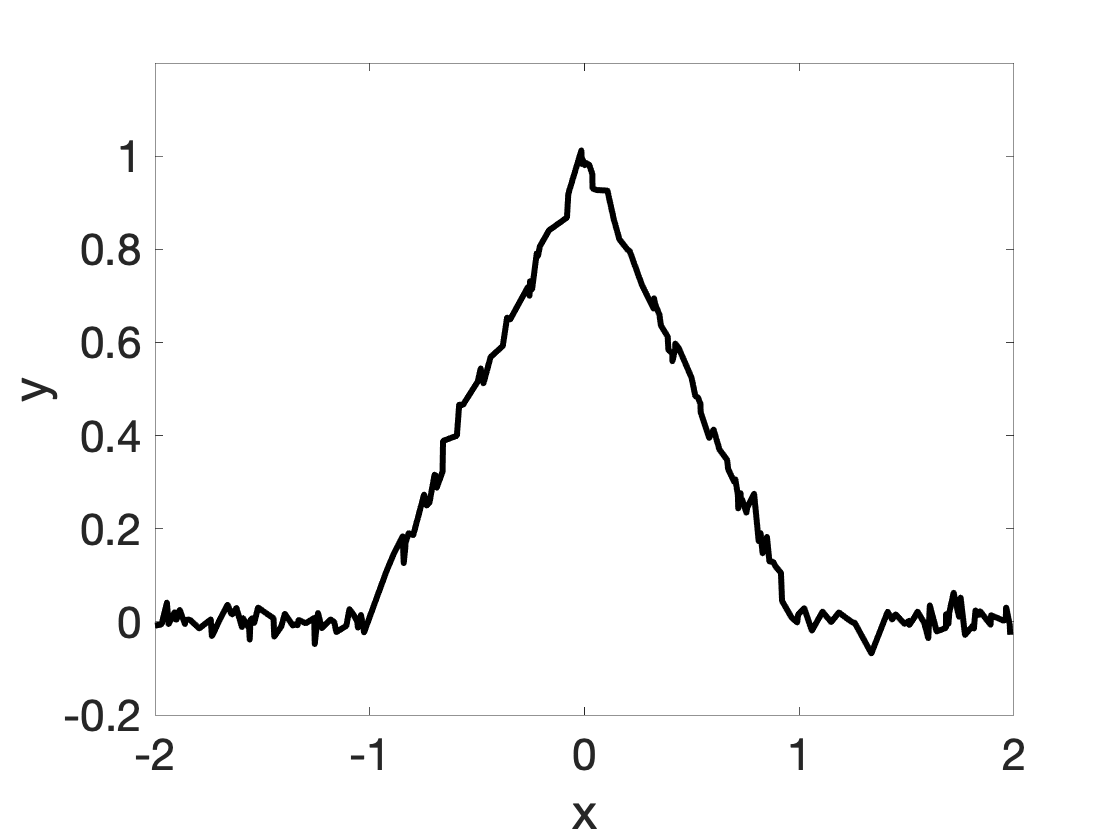} 
\includegraphics[width=1.5in]{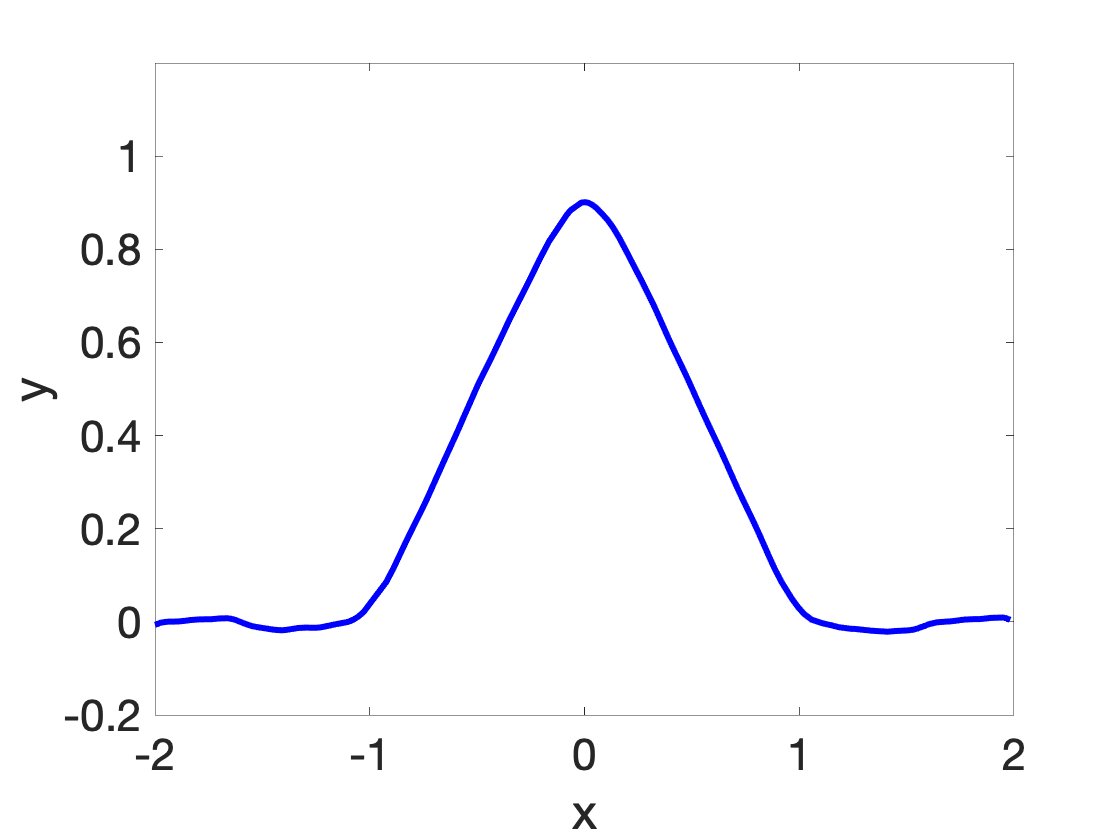} 
\includegraphics[width=1.5in]{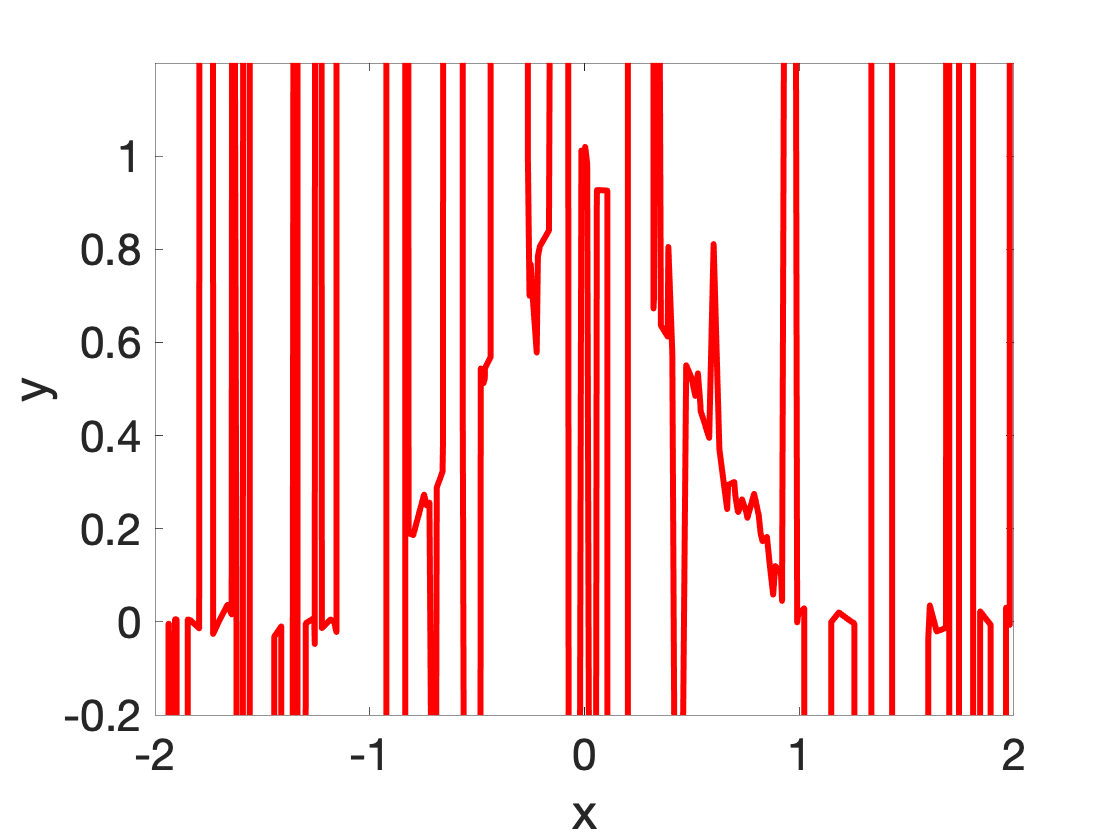} 
\caption{\small \textbf{Comparison, Noise}: The first column includes the Runge function (top) and a triangle function (bottom) each with $5\%$ relative noise. The second and third columns are the approximations via the SRFE and the OLS methods respectively with the same 200 randomly sampled points.  
}
\label{Fig:Comparison1d2}
\end{center}
\end{figure*}

\subsection{Low Order Approximations}

\begin{table*}[t]
\centering 
\footnotesize
\begin{tabular}{|c||c|c|c|c|c|c|}\hline
$f(\x)$&   $\sigma$  & $d$ & $q=1$ & $q=2$ & $q=3$ & $q=5$   \\ \hline   \hline 
$\left(\sum_{i=1}^d x_j\right)^2$ 
& $0.1$
& $1$
&{$0.82$}
&\color{purple}{$5.71\times 10^{-6}$}
&$6.92\times 10^{-5}$
&$8.3\times 10^{-4}$
\\ \hline 
$(1+\|\x\|_2^2)^{-1/2}$ 
&1
& 5
&$3.27$
&$1.60$
&$1.95$
&\color{purple}{$1.72$}
\\ \hline
$\sqrt{1+\|\x\|_2^2}$ 
&$1$
& 5
&$1.02$
&$0.73$
&$0.80$
&\color{purple}{$1.10$}
\\ \hline
$\mathrm{sinc}(x_1)\mathrm{sinc}(x_3)^3+\mathrm{sinc}(x_2)$
&$\pi$
& $5$
&$12.90$
&\color{purple}{$1.19$}
&$1.13$
&$3.51$
\\ \hline
$\frac{x_1x_2}{1+x_3^6}$
&$1$
&$5$
&$100.30$
&$21.53$
&\color{purple}{$4.95$}
&$5.06$
\\ \hline
$\sum_{i=1}^{d}\exp(-|x_i|)$
&$1$
&$100$
&\color{purple}{$0.91$}
&$1.43$
&$1.57$
&$1.96$
\\ \hline
\end{tabular}
\caption{\small \textbf{Low Order Examples} The table contains the relative test error (as a percentage) for approximating various functions using different $q$ values. The purple values represent the order of the function.  We fix $m=1000$ and $N=10000$ with random sine features.  We draw $\x\sim \mathcal{U}[-1, 1]^d$ and the nonzero values of $\w$ are drawn from $\mathcal{N}(\mathbf{0}, \sigma^2)$.}\label{tablevarious}
\end{table*}

In \Cref{tablevarious}, we test the effect of varying $q$ for different functions using \Cref{alg:B2} and recording the relative errors. The highlighted (purple) values represent the explicit order of the function. We fixed $m=1000$, $N=10000$ and used the random sine features.  The data is sampled from $\mathcal{U}[-1, 1]^d$ and the nonzero values of $\w$ are drawn from $\mathcal{N}(\mathbf{0}, \sigma^2)$, where $\sigma$ and $d$ are included in the table for each example.

In the second and third examples, while the functions are order $q=d$ functions, they enjoy better accuracy for $q=2$. This could be due to several phenomena. The first is that, with fixed $m$ and $N$, the error may increase as $q$ increase (see \Cref{thm:low_order}). However, this should partially be mitigated since we chose $N=10000$ large enough. Another reason is that, with respect to some expansion (i.e. Fourier or Taylor), the functions can be written as an order $q<d$ function within some level of accuracy. This motivates further investigations in future work. The other examples show a clear transition when the correct range for $q$ is obtained.

\subsection{HyShot 30 Data}
\begin{table*}[t]
\footnotesize
\centering
\begin{tabular}{|c||c|c|c|c|}\hline
\textbf{HyShot 30} &   $N=100$ & $N=200$ & $N=400$ & $N=800$ 
\\ \hline 
SRFE with Sine
&$6.95$
&$6.23$
&$5.76$
&$5.64$
\\ \hline
SRFE with ReLU
&$1.40$
&$1.45$
&$1.51$
&$1.59$
\\ \hline
Random Fourier Features
& $84.23$
& $89.99$
& $ 95.17$
& $97.84$
\\ \hline
Two-layer ReLU Network
& $7.29$
& $11.50$
& $ 11.19$
& $11.33$
\\ \hline
\textbf{NACA Sound} &   $N=250$ & $N=1500$ & $N=5000$ & $N=10000$ 
\\ \hline 
SRFE (Train)
&$3.22$
&$2.30$
&$2.30$
&$2.31$
\\ \hline
SRFE (Test)
&$3.22$
&$3.04$
&$2.77$
&$2.78$
\\ \hline
SRFE (Average Sparsity)
&$250$
&$364.4$
&$185.7$
&$185.7$
\\ \hline
Random Fourier Features (Train)
& $3.22$
& $0.25$
& $ 0.20$
& $0.19$
\\ \hline
Random Fourier Features (Test)
& $7.45$
& $ 2.13\times 10^{8}$
& $ 1.69\times 10^{8}$
& $1.48\times 10^{8}$
\\ \hline
\end{tabular}
\caption{\small \textbf{HyShot 30 and NACA Sound Datasets}: Average relative train and test errors over 10 random trials (as a percentage). For the shallow NN, we choose the hidden layer so that the total number of parameters match $N$.}
\label{TableComparison}
\vspace{-6mm}
\end{table*}
In \Cref{TableComparison}, we apply the SRFE on the HyShot dataset (Hypersonics Flow Data \cite{constantine2015exploiting}) and measure the relative testing error as a function of $N$ (the number of random features).  The input space is $d=7$ dimensional and the dataset includes $52$ total samples (which we split into 26-26). We set  $\eta=0.01$, $\sigma=2\pi$, $p\sim \mathcal{U}[0,1]$, and $q=7$ (no coordinate sparsity is assumed).  In this setting, we have $N\gg m$, which causes the RFF model and the two-layer ReLU network to overfit on the data (the training loss is small). When using $\phi(\x;\w)=\sin(\langle \x,\w\rangle)$, the SRFE produces consistent testing error which decreases as $N$ increases. On the other hand, when  $\phi(\x;\w)=\text{ReLU}(\langle \x,\w\rangle)$, the results using SRFE achieve a smaller overall testing error but do not improve with $N$.  \Cref{TableComparison} shows that unlike the SRFE,  no gains are made from increasing the number of trainable parameters in the shallow NN model.

\subsection{NACA Sound Dataset}
We comparing the SRFE and the RFF models without coordinate sparsity on the National Advisory Committee for Aeronautics (NACA) sound dataset \cite{Dua:2019} and measure the relative training and testing error as a function of $N$. The input space is $d=5$ dimensional, the total number of samples is $1503$, the train-test split $80-20$, $\eta=100$, $\sigma = 1$, and $p\sim \mathcal{U}[0,1]$.
The relative testing errors in \Cref{TableComparison} indicate an overall consistent result, in terms of the coefficient sparsity and the errors, when using the SRFE approach. The RFF model overfits as $N$ increases beyond the size of the training set. 
\subsection{Comparison with Sparse PCE}
In \Cref{Fig:ErrorPCE},  we compare the SRFE-S approach with the Sparse PCE approach \cite{marelli2014uqlab} using various random sampling methods on the Ishigami example $f(x_1, x_2, x_3)=\sin(x_1) + 7 \sin^2(x_2) + 0.1 x_3^4 \sin(x_1)$
which is of order 2.
\begin{figure}[!t]
\begin{center}
\includegraphics[width=1.5in]{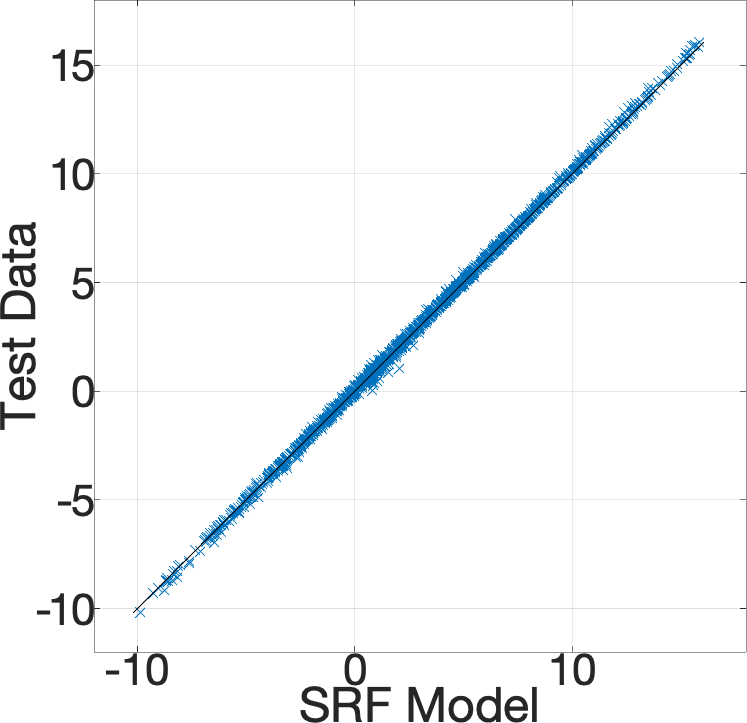} 
\includegraphics[width=1.55in]{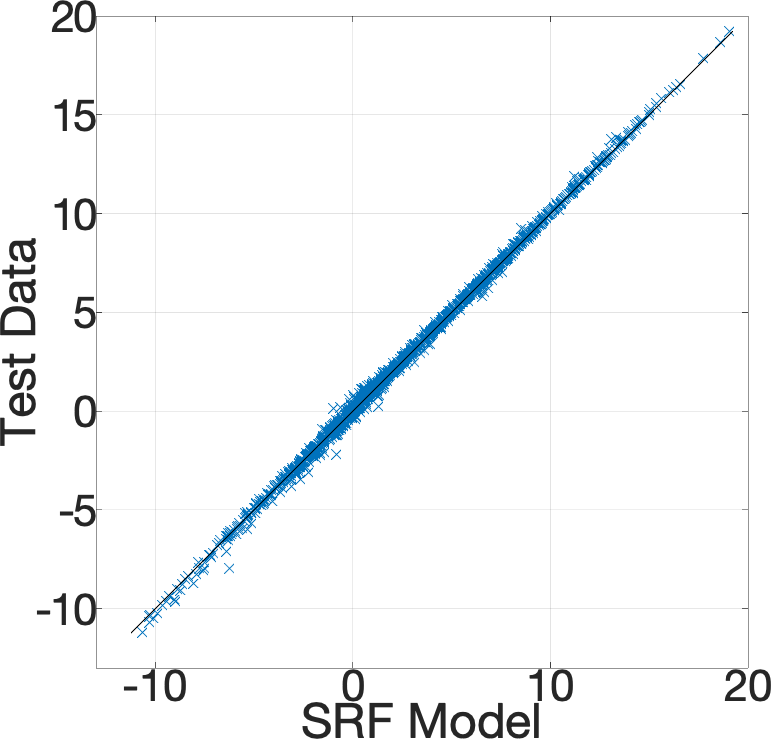}
\includegraphics[width=1.5in]{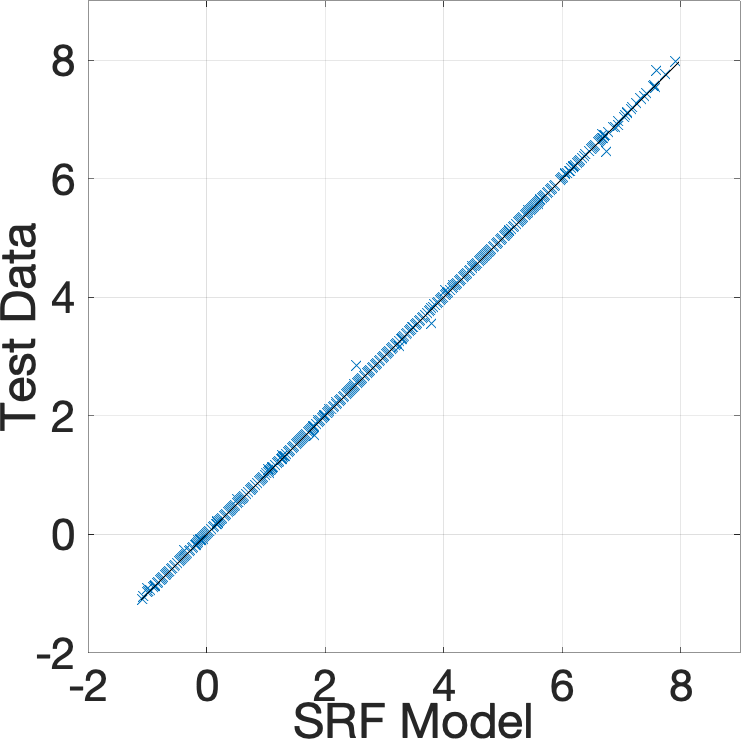}\\
\includegraphics[width=1.5in]{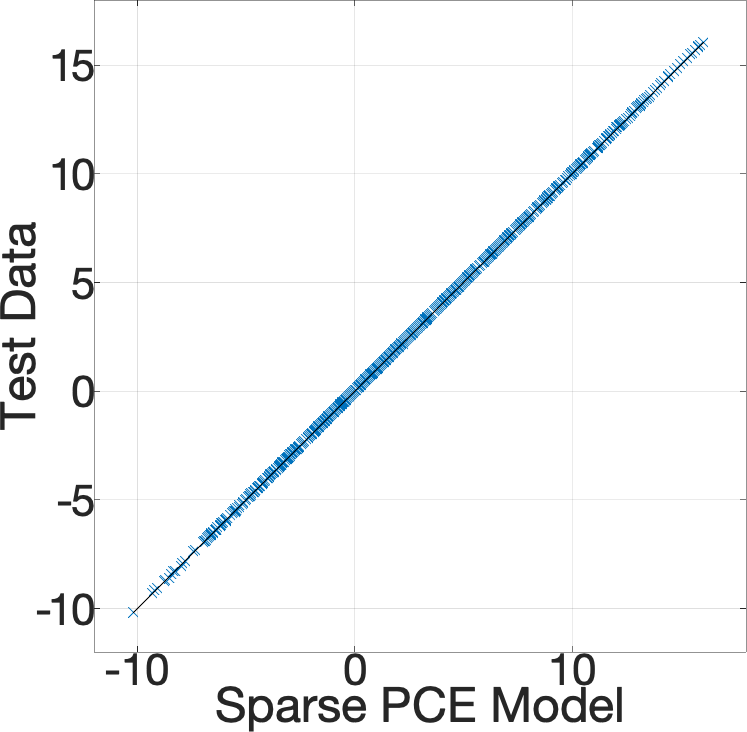}
\includegraphics[width=1.55in]{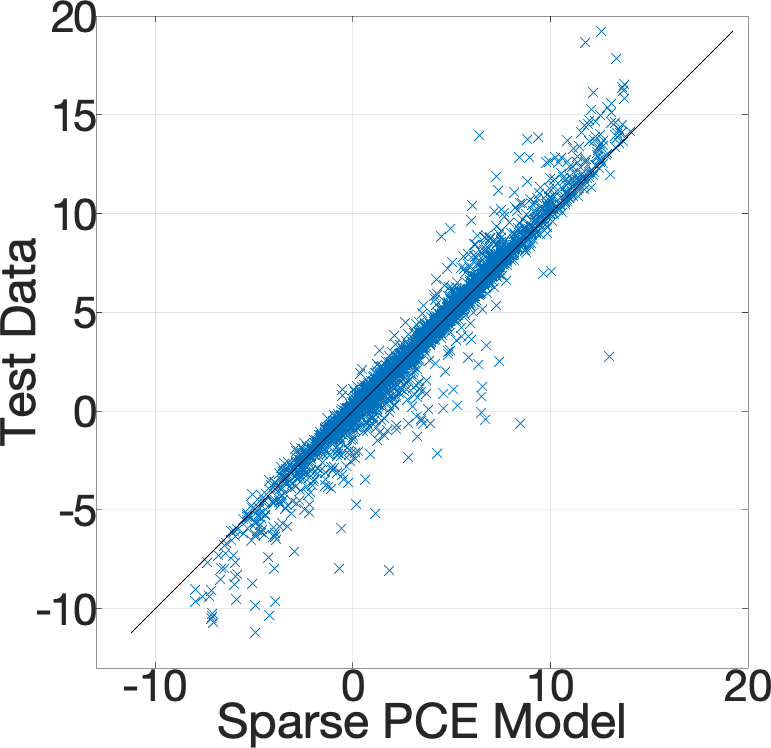}
\includegraphics[width=1.5in]{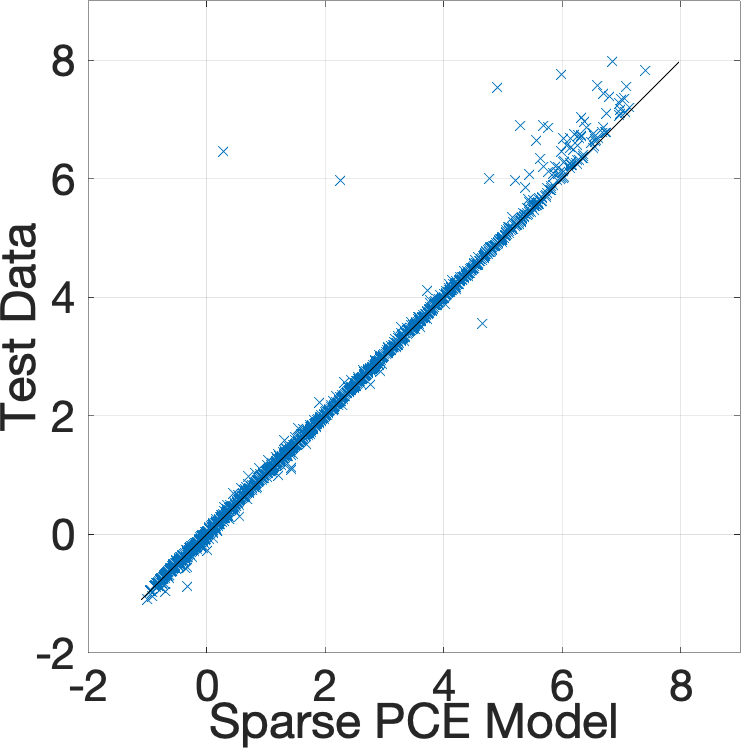}
\caption{\small \textbf{Comparison with Sparse PCE.} Each scatter plot is the model response versus the data. The first row is the SRFE and the second row is the Sparse PCE model. The first column uses i.i.d. samples from $\mathcal{U}[-\pi, \pi]^d$, the third column uses i.i.d. samples from  $\N\left(0,\frac{1}{4}\I_d\right)$, and the second column uses the sum of samples from  $\N\left(0,\frac{1}{100}\I_d\right)$ and $\mathcal{U}[-\pi, \pi]^d$. Each model uses $N=3276$ features (which is equivalent to a degree-25 polynomial system in the case of the Sparse PCE approach) and (the same) $m=200$ random samples. While the Sparse PCE performs well on the uniform distribution (first row), the SRFE produces accurate approximations in all cases.
}
\label{Fig:ErrorPCE}
\end{center}
\vspace{-6mm}
\end{figure}
The first column of \Cref{Fig:ErrorPCE} uses i.i.d. samples $\x_k\sim \mathcal{U}[-\pi, \pi]^d$, the third column uses i.i.d. samples $\x_k \sim \N\left(0,\frac{1}{4}\I_d\right)$, and the second column uses a mixed distribution $\x_k=\x_{k,1}+\x_{k,2}$ where  $\x_{k,1}\sim \N\left(0,\frac{1}{100}\I_d\right)$ and $\x_{k,1}\sim \mathcal{U}[-\pi, \pi]^d$. Each model uses $N=3276$ features (which is equivalent to a degree-25 polynomial system in the case of the Sparse PCE approach) and (the same) $200$ random samples. The hyper-parameters for the SRFE-S are set to $q=2$, $\sigma=\frac{3\pi}{2}$, and $p\sim \mathcal{U}[0,2\pi]$. When using uniformly random samples, the Sparse PCE approach produces lower testing error ($0.24\%$ versus $1.43\%$), which continues to perform well as $N$ increases. This is due in part to the fact that the orthogonal polynomial basis (in this case, the Legendre basis) has knowledge of the input distribution. When the samples are Gaussian, the SRFE produces a more accurate solution than the Sparse PCE method ($0.44\%$ versus $6.24\%$). For the mixture case, the SRFE outperforms the Sparse PCE method ($2.11\%$ versus $15.05\%$). Note that the Sparse PCE must derive the orthogonal basis from the data (or use the Legendre basis as its default), where as, at least experimentally, our approach is applicable to a larger class of input distributions. 
\section{Conclusion}\label{sec:conc}
We proposed the sparse random features method as a new approach in function approximation. For low order functions, i.e. functions that admit a decomposition to terms depending on only a few of the independent variables, we introduce low order random features.  By utilizing techniques from compressive sensing and probability, we provided generalization bounds for the proposed scheme and established sample and feature complexities. On several examples, we showed improved accuracy over other popular approximation schemes. As part of the future work, we intend to explore the avenues to incorporate additional functional structures into the proposed framework with the hope of further improving the approximation properties of the proposed scheme. In addition, by considering random features within a ridge regression approach, \cite{rudi2017generalization} showed that the computational gains of random features come at the expense of learning accuracy, $N = \O(\sqrt{m}\log m)$ features are sufficient for $\O(1/\sqrt{m})$ error, where $m$ is the number of samples. Utilizing this result in our proposed framework is an interesting direction which is left for future work.

\section*{Acknowledgments}

We thank Zhijun Chen, Jiannan Jiang, Kameron Harris, Andrea Montanari, Rene Vidal, and  Yuege Xie for their helpful feedback which led to significant improvements.

\newpage
\appendix
\vspace{6mm}
\begin{center}
    \textbf{\Large Supplementary Material} \vspace{4mm}
\end{center}

\section{Useful Tools and Definitions}
\begin{definition}[\textbf{Normal Distribution}]\label{def:normal}
The random vector $\w = (w_1, \dots, w_d)$ is called a centered normal random vector with mean zero and variance $\sigma^2$ if it has density function\footnote{Note that we only consider the homoscedastic case.}
\begin{equation}
\rho(\w) = (2\pi \sigma^2)^{-d/2} \exp\left(-\frac{\| \w \|^2}{2\sigma^2}\right)  .  
\end{equation}
Furthermore, it holds that $\mathbb{E} \| \w \|^2 = d\sigma^2$, and by a standard concentration argument for any $0<t<1$
\begin{equation}\label{eq:gauss_concentrate}
\P((1-t)d\sigma^2\leq\|\w\|^2\leq(1+t)d\sigma^2) \geq 1 - 2\exp\left(-\frac{ t^2d}{12}\right)
\end{equation}
\end{definition}
\begin{lemma}[\textbf{Khintchine Inequality \cite{foucartmathematical}}]\label{lem:Rademacher}
If $\{\xi_j\}_{j=1}^N$ is a collection of i.i.d. Rademacher random variables, for any $q_1,\dots,q_N \in \C$, and $0<p\leq2$,
\begin{equation}
\left(\E_\xi\left[\left|\sum_{j=1}^N \xi_j q_j\right|^p\right]\right)^{\frac{1}{p}} \leq \sqrt{\sum_{j=1}^N|q_j|^2}.
\end{equation}
\end{lemma}
\begin{lemma}[\textbf{Rademacher Complexity \cite{foucartmathematical}}]\label{lem:Rademacher_complexity}
Assume that $\{\v_j\}_{j=1}^M$ is a sequence of independent random
vectors in a finite-dimensional vector space $V$ with norm $\| \cdot\|$. Let $F: V \rightarrow \R$ be a convex function. Then
\begin{equation}
\E_{\v} F\left(\sum_{j=1}^M\v_j - \E[\v_j]\right)\leq \E_{\xi,\v} F\left(2\sum_{j=1}^M\xi_j\v_j\right),
\end{equation}
where $\{\xi_j\}$ is a Rademacher sequence independent of $\v$.
\end{lemma}
\begin{lemma}[\textbf{Contraction of Rademacher Complexity \cite{ledoux2013probability,bartlett2002rademacher}}]\label{lem:rademacher_lip}
Let $\{\phi_j\}_{i=1}^N$ be a collection of real-valued functions defined on $\R$ that are $L$-Lipschitz and satisfy $\phi_j(0) = 0$. Then,
\begin{equation}
\E_\xi \sup_{x \in \X} \left|\sum_{j=1}^N \xi_j \phi_j(x_j)\right| \leq 2L \E_\xi\sup_{x \in \X} \left|\sum_{j=1}^N \xi_j x_j\right|
\end{equation}
for any bounded subset $\X$ in $\R^N$. Here, $\{\xi_j\}$ is a Rademacher sequence independent of $\{x_j\}$.
\end{lemma}

\begin{lemma}[\textbf{Stability of BP-based Sparse Reconstruction \cite{foucartmathematical}}]\label{lem:l1_book}
Let $\A \in \C^{m\times N}$ be a matrix with coherence $\mu_\A$. If the coherence of $A$ satisfies
\begin{equation}
\mu_\A    \leq \frac{4}{\sqrt{41}(2s-1)},
\end{equation}
then,  for any  vector $\c^\star \in \C^N$ satisfying ${\bf y} = \A\c^\star+\e$ with $\|\e\| \leq \eta\sqrt{m}$, a minimizer $\c^\sharp$ of the BP method \eqref{eq:BP} approximates the vector $\c^\star$ with the error bounds
\begin{equation}
\begin{aligned}
\|\c^\star-\c^\sharp\|_2&\leq C' \, \frac{\kappa_{s,1}(\c^\star)}{\sqrt{s}}+C\eta\\
\|\c^\star-\c^\sharp\|_1&\leq C'\kappa_{s,1}(\c^\star)+C\sqrt{s}\eta, 
\end{aligned}
\end{equation}
where $C,C'>0$ are constants.
\end{lemma}

\begin{lemma}[\textbf{Stability of Threshold BP-based Sparse Reconstruction}]\label{lem:thres_l1}
Let $\A \in \C^{m\times N}$ be a matrix with coherence $\mu_\A$. If the coherence of $A$ satisfies
\begin{equation}
\mu_\A    \leq \frac{4}{\sqrt{41}(2s-1)},
\end{equation}
then,  for any  vector $\c^\star \in \C^N$ satisfying ${\bf y} = \A\c^\star+\e$ with $\|\e\| \leq \eta\sqrt{m}$, a minimizer $\c^\sharp$ of the BP method \eqref{eq:BP} approximates the vector $\c^\star$ with the error bounds
\begin{equation}
\begin{aligned}
\|\c^{\sharp}|_{\S^\sharp}-\c^\star\|_2&\leq C' \, \frac{\kappa_{s,1}(\c^\star)}{\sqrt{s}}+C\eta+4\kappa_{s,2}(\c^\star)
\end{aligned}
\end{equation}
where $C,C'>0$ are constants, and $\S^\sharp$ is defined as the support set of the $s$ largest (in magnitude) coefficients of $\c^\sharp$. Note that $\kappa_{s,2}(\c^\star) \leq \frac{\|\c^\star\|_1}{2\sqrt{s}}$.
\end{lemma}

\begin{proof}
Let $\S^\star$ be the support set of the $s$ largest (in magnitude) coefficients of $\c^\star$ and $\S^\sharp$ be the support set of the $s$ largest (in magnitude) coefficients of $\c^\sharp$. Note that $|\S^\star|=s=|\S^\sharp|$ and $|\S^\star\setminus \S^\sharp|=|\S^\sharp\setminus\S^\star|$. 

First, we can bound the difference between the threshold solution to the basis pursuit problem and $c^\star$ by 
\begin{equation}
\begin{aligned}
\|\c^{\sharp}|_{\S^\sharp}-\c^\star\|_2&\leq \|\c^{\sharp}|_{\S^\sharp}-\c^\star|_{\S^\star}\|_2+\|\c^\star|_{[N]\setminus\S^\star}\|_2\\
&\leq \|\c^{\sharp}|_{\S^\sharp}-\c^\star|_{\S^\star}\|_2+\kappa_{s,2}(\c^\star).
\end{aligned}
\end{equation}
Then using a similar argument as in the proof of Corollary 3.2 \cite{needell2010signal}, we have
\begin{equation}
    \|\c^{\sharp}|_{\S^\sharp}-\c^\star|_{\S^\star}\|_2 \leq 3\|\c^{\sharp}-\c^\star|_{\S^\star}\|_2.
\end{equation}
Therefore, by \Cref{lem:l1_book}
\begin{equation}
\begin{aligned}
\|\c^{\sharp}|_{\S^\sharp}-\c^\star\|_2
&\leq 3\|\c^{\sharp}-\c^\star|_{\S^\star}\|_2+\kappa_{s,2}(\c^\star)\\
&\leq 3\|\c^{\sharp}-\c^\star\|_2+4\kappa_{s,2}(\c^\star)\\
&\leq 3C' \, \frac{\kappa_{s,1}(\c^\star)}{\sqrt{s}}+3C\eta+4\kappa_{s,2}(\c^\star).
\end{aligned}
\end{equation}
\end{proof}

\begin{lemma}[\textbf{Samples Lie in the Domain}]\label{lem:gaussian_domain}
Suppose that $\x_1, \dots, \x_m \sim {\cal N}(0, \gamma^2 \I_d)$ are i.i.d. Gaussian points and thus $\mathbb{E} \| \x \|^2 = d\gamma^2$.  Let $R>0$ be a fixed radius. Then, for any $0<\delta<1$, 
the probability of all $m$ samples $\x_1, \dots, \x_m \in  \B^d(R)$ is at least $1-\delta$ provided that:
\begin{equation}
R\geq \gamma\ \sqrt{d+\sqrt{12d\log\left(\frac{m}{\delta}\right)}}.
\end{equation}
\end{lemma}
\begin{proof}
Straightforward using the concentration result in \Cref{def:normal} and the union bound.
\end{proof}

\section{Important Lemmata}


The proofs of the generalization bounds utilize the following lemmata. The first establishes a bound between the target function and the best $\phi$-approximation, which is needed to control the stability parameter in the basis pursuit problem. 

\begin{lemma} \label{errorboundlemma1}

Fix confidence parameter $\delta > 0$ and accuracy parameter $\epsilon > 0$.  Recall the setting of \Cref{alg:B} and suppose $f \in \F(\phi,\rho)$ where $\phi(\x;\w) = \exp(i \langle \x,\w \rangle)$ and $\rho$ is the probability density function (with finite second moment) used for sampling the random weights $\w$.
Consider a set $\X \subset \R^d$ with diameter $R = \sup\limits_{\x \in \X} \| \x \|$.
Suppose 
\begin{equation}
N \geq  \frac{4}{\epsilon^2}\left(1 + 4 R \sqrt{\E \| \w \|^2} + \sqrt{\frac{1}{2}\log\left(\frac{1}{\delta}\right)}\right)^2.
\end{equation}
Consider the random feature approximation
\begin{equation}
f^{\star}(\x) := \sum_{j =1}^N c^{\star}_j \, \exp({i \langle \x, \w_j \rangle}), \quad \text{where} \ \ c^{\star}_j := \frac{\alpha(\w_j)}{N\rho(\w_j)}.
\end{equation}
Then, with probability at least $1-\delta$ with respect to the draw of the $\w$'s, the following holds
\begin{equation}
\sup_{\x \in \X} | f(\x) - f^{\star}(\x) | \leq \epsilon \|f\|_\rho.
\end{equation}
\end{lemma}
\begin{proof}
First, by construction we have $|c_j^\star|\leq \dfrac{\|f\|_\rho}{N}$ and, for fixed $\x$, $\E_{\w}[f^\star(\x)] = f(\x)$. Define the random variable \[v(\w_1,\dots,\w_N) := \left\| f - f^{\star} \right\|_{L^\infty(\X)} = \sup_{\x \in \X} |f(\x) - f^{\star}(\x)| = \sup_{\x \in \X } \left|  \mathbb{E}_\omega \left[f^{\star}(\x)\right]-f^{\star}(\x)  \right|.\] 
Following \cite{rahimi2008uniform},  we prove the lemma using  McDiarmid's inequality.

First, observe that $v$ is stable under perturbation of any one of its coordinates. Specifically, we have that for the $k\ts{th}$ coordinate (others kept fixed)
\begin{equation}
\begin{aligned}
|v({\w_{k}})-v(\tilde{\w}_k)| &\leq \frac{1}{N} \sup_{\x\in X}  \left|\frac{\alpha({\w_{k}})}{\rho({\w_{k}})}\phi(\x;{\w_{k}})-\frac{\alpha(\tilde{\w}_k)}{\rho(\tilde{\w}_k)}\phi(\x;\tilde{\w}_k)\right| \\&\leq \frac{2\|f\|_\rho}{N}=:\Delta_v,
\end{aligned}
\end{equation}
where we used the triangle inequality for the $\| \cdot \|_{L^{\infty}}$ norm, and uniform bounded on $\phi$. 

We would like to apply McDiarmid's concentration inequality:
$\P(v\geq \E[v]+t)\leq \exp(-\frac{2t^2}{N \Delta_v^2}),$ which requires us to estimate the expectation of $v$. To do this, following \cite{rahimi2008uniform,rahimi2008weighted} we exploit properties of Rademacher random variables \cite{bartlett2002rademacher}. Using the triangle inequality and Lemma \ref{lem:Rademacher_complexity} yields 
\begin{equation}
\begin{aligned}
  \E_{\w}[v] &\leq 2 \E_{\w,\xi} \sup_{\x \in \X} \left|\sum_{j=1}^N \xi_j c_j^\star \phi(  \x;\w_j )\right| 
   \\&= 2 \E_{\w,\xi} \sup_{\x \in \X} \left|\sum_{j=1}^N \xi_j c_j^\star \left( \phi( \x;\w_j ) - \phi(0)+ 1 \right)\right| \\
  &\leq 2 \E_{\w,\xi} \sup_{\x \in \X} \left|\sum_{j=1}^N \xi_j c_j^\star \left( \phi( \x;\w_j ) - \phi(0) \right) \right|+ 2\E_{\xi}\left|\sum_{j=1}^N \xi_j c_j^\star\right|,
\end{aligned}
\end{equation}
noting that $\phi(0) = 1$. The second term above can be bounded using \Cref{lem:Rademacher} and recalling that $|c_j^\star|\leq \|f\|_\rho/N$:
\begin{equation}
\begin{aligned}
   2\E_{\xi}\left|\sum_{j=1}^N \xi_j c_j^\star\right| &\leq 2\sqrt{\sum_{j=1}^N |c_j^\star|^2 }  \leq \frac{2\|f\|_\rho}{\sqrt{N}}.
\end{aligned}
\end{equation}
We now bound the first term. Let $c_j^\star := |c_j^\star|  \exp({ i\theta_j})$ and note that by Euler's formula 
\[c_j^\star \phi(  \x;\w_j ) = |c_j^\star|\cos(  \langle \x,\w_j \rangle + \theta_j) +i |c_j^\star| \sin(  \langle \x,\w_j \rangle+ \theta_j).\]
Therefore,
\begin{equation}
\begin{aligned}
&\E_{\w,\xi} \sup_{\x \in \X} \left|\sum_{j=1}^N \xi_j c_j^\star \left( \phi( \x;\w_j ) - \phi(0) \right) \right|\\&\leq  \E_{\w,\xi} \sup_{\x \in \X}\bigg| \sum_{j=1}^N \xi_j  |c_j^\star| (\cos(  \langle \x,\w_j \rangle + \theta_j) 
-\cos( \theta_j) +i \sin(  \langle \x,\w_j \rangle+ \theta_j)-i\sin( \theta_j))\bigg|\\
&\leq  \E_{\w,\xi} \sup_{\x \in \X}\left| \sum_{j=1}^N \xi_j  |c_j^\star| \left(\cos(  \langle \x,\w_j \rangle + \theta_j)-\cos( \theta_j)\right) \right|\\&\qquad+ \E_{\w,\xi} \sup_{\x \in \X}\left| \sum_{j=1}^N \xi_j  |c_j^\star| \left(\sin(  \langle \x,\w_j \rangle + \theta_j)-\sin( \theta_j)\right)\right|
\end{aligned}
\end{equation}
The functions $|c_j^{\star}| \left(\cos( \cdot+\theta_j)-\cos( \theta_j)\right)$ and $|c_j^{\star}| \left(\sin( \cdot+\theta_j)-\sin( \theta_j)\right)$ are $\|f \|_{\rho}/N$-Lipschitz 
and are zero at zero. Thus, using \Cref{lem:rademacher_lip} we may write,
\begin{equation}
\begin{aligned}
2 \E_{\w,\xi} \sup_{\x \in \X}\left|\sum_{j=1}^N \xi_j c_j^{\star} (\phi(  \x;\w_j ) - \phi(0))\right| 
&\leq \frac{8\|f\|_{\rho}}{N} \E_{\w,\xi} \sup_{\x \in \X} \left|\sum_{j=1}^N \xi_j  \langle \x,\w_j \rangle\right|\\
&\leq \frac{8\|f\|_{\rho}}{N} \sup_{\x \in \X} \|\x\| \text{ }\E_{\w,\xi}\left\|\sum_{j=1}^N\xi_j\w_j\right\| \\&\leq \frac{8\|f\|_{\rho}R}{N}  \text{ }\E_{\w,\xi}\left\|\sum_{j=1}^N\xi_j\w_j\right\|
\end{aligned}
\end{equation}
where we used the Cauchy-Schwartz inequality to establish the second inequality. Next, note that by Jensen's inequality and \Cref{lem:Rademacher}
\begin{equation}
\begin{aligned}
\E_{\w,\xi}\left\|\sum_{j=1}^N\xi_j\w_j\right\| &\leq \sqrt{\E_{\w,\xi}\left\|\sum_{j=1}^N\xi_j\w_j\right\|^2}
\leq \sqrt{N \E_{\w}\|\w\|^2}.
\end{aligned}
\end{equation}
Altogether, we have the bound
\begin{equation}
\mathbb{E}[v] \leq \frac{2 \| f \|_{\rho} (1 + 4 R \sqrt{ \E\|\w\|^2})}{\sqrt{N}} =:M_v.
\end{equation}
We are now in a position to apply McDiarmid's concentration inequality to obtain
\begin{equation}
\P(v\geq M_v+t)\leq\P(v\geq \E[v]+t)\leq \exp\left(-\frac{2t^2}{N \Delta_v^2}\right).
\end{equation}
By setting $M_v+t \leq \epsilon \|f\|_\rho$ and the probability bound on the right hand side of the equation above to $\delta$, we solve for $N$ and $t$ to obtain
\begin{equation}
t = \|f\|_\rho\sqrt{\frac{2}{N}\log\left(\frac{1}{\delta}\right)},
\end{equation}
and 
\begin{equation}
N \geq \frac{4 }{\epsilon^2}\left(1 + 4 R \sqrt{ \E\|\w\|^2} + \sqrt{\frac{1}{2}\log(1/\delta)}\right)^2
\end{equation}
which completes the proof.
\end{proof}

The next result provides a coherence estimate on the random feature matrix.

\begin{lemma}[\textbf{Coherence Analysis}]\label{lem:coherence_analysis}
Consider a complete set of $q$-sparse feature weights in $\R^d$, $\w_1,\dots, \w_N$ drawn from $\mathcal N(\mathbf{0}, \sigma^2 \I_q)$ and a set of data samples $\x_1,\dots, \x_m\sim \mathcal{N}(\mathbf{0}, \gamma^2 \I_d)$.  Define the random features $\phi(\x; \w)=\exp(i\langle \x, \w\rangle )$ and let $\A\in \C^{m\times N}$ denote the associated  random feature matrix where $a_{k, j}=\phi(\x_k;\w_j)$.  For a fixed $0<\delta<1$ and for some integer $s\geq 1$, suppose
\begin{align}
m& \geq 4(2\gamma^2\sigma^2+1)^{\max\{2q-d,0\}} (\gamma^2\sigma^2+1)^{\min\{2q,2d-2q\}}\log \frac{N^2}{\delta}  \\
\gamma^2\sigma^2&\geq \frac{1}{2}\left(\left(\frac{\sqrt{41}(2s-1)}{2}\right)^{\frac{2}{q}}-1\right),
\end{align}
then we have with probability at least $1-\delta$, that the coherence of $\A$ is bounded by
\begin{equation}
\mu_{\A}\leq \frac{4}{\sqrt{41}(2s-1)}.
\end{equation}
\end{lemma}
\begin{proof}
Let ${\a_j}, \a_\ell$ denote two columns of $\A$.  Let ${\S_j}, \S_\ell$ denote the support sets of ${\w_j}, \w_\ell$, respectively, and let $\mathcal{G}={\S_j}\cap \S_\ell$.  Then using the characteristic function of the Gaussian distribution
\begin{equation}\label{eq:main_expectation_lem2}
\begin{aligned}
\E[\langle {\a_j}, \a_\ell \rangle&\mid {\S_j}, \S_\ell]
=\E_{{\w_j}, \w_\ell}\left[\E_{{\x_k}}\left[\sum_{j=1}^m \exp(i\langle {\w_j}-\w_\ell, {\x_k} \rangle)\mid {\w_j}, \w_\ell\right]\mid {\S_j}, \S_\ell\right]\\
&=\E_{{\w_j}, \w_\ell}\left[m \exp\left(-\frac{\gamma^2}{2}\|{\w_j}-\w_\ell\|^2\right)\mid {\S_j}, \S_\ell\right]\\
&=m \E_{{\w_j}, \w_\ell}\left[ \exp\left(-\frac{\gamma^2}{2}({\w_j}-\w_\ell)_1^2\right)\dots \exp\left(-\frac{\gamma^2}{2}({\w_j}-\w_\ell)_d^2\right)\mid {\S_j}, \S_\ell\right]\\ 
&=m \left(\frac{1}{\sqrt{2\gamma^2\sigma^2+1}}\right)^{|\mathcal{G}|}\left(\frac{1}{\sqrt{\gamma^2\sigma^2+1}}\right)^{2q-2|\mathcal{G}|}=:m{\Gamma_{j\ell}},
\end{aligned}
\end{equation}
where $0\leq\max\{2q-d,0\}\leq|\mathcal{G}|\leq q\leq d$. Assuming that $q<\frac{d}{2}$
we have the following bound
\begin{equation}
\left(\frac{1}{\sqrt{\gamma^2\sigma^2+1}}\right)^{2q}=:\Gamma_{\text{min}}\leq\Gamma_{j,\ell}\leq \left(\frac{1}{\sqrt{2\gamma^2\sigma^2+1}}\right)^{q},
\end{equation}
using the inequality $2\gamma^2\sigma^2+1\leq \left(\gamma^2\sigma^2+1\right)^2$.
Given that ${\x_k}$'s are i.i.d., applying the Bernstein's inequality and recalling that $\mu_{j\ell}=\frac{ \langle {\a_j},  \a_\ell \rangle}{ m}$, yields
\begin{equation}
\begin{aligned}
\P(|{\mu_{j\ell}}-{\Gamma_{j\ell}}|\geq {\Gamma_{j\ell}}\mid {\S_j}, \S_\ell)\leq 2\exp\left(-\frac{\frac{1}{2}m^2{\Gamma_{j\ell}}^2}{m+\frac{2}{3}m{\Gamma_{j\ell}}}\right).
\end{aligned}
\end{equation}
Since ${\Gamma_{j\ell}}\leq \frac{3}{2}$, the last result simplifies to 
\begin{equation}
\begin{aligned}
\P(|{\mu_{j\ell}}-{\Gamma_{j\ell}}|\geq {\Gamma_{j\ell}}\mid {\S_j}, \S_\ell)
\leq 2\exp\left(-\frac{1}{4}m{\Gamma_{j\ell}}^2\right)
\leq 2\exp\left(-\frac{1}{4}m\Gamma_{min}^2\right).
\end{aligned}
\end{equation}
Taking a union bound over all ${\genfrac(){0pt}{2}{N}{2}}\leq \frac{N^2}{2}$ pairs of columns implies that
\begin{equation}
\P(\exists\,k, \,\ell\textrm{ s.t. }\,|{\mu_{j\ell}}-{\Gamma_{j\ell}}|\geq  {\Gamma_{j\ell}})\leq N^2\exp\left(-\frac{1}{4}m\Gamma_{min}^2\right)\leq \delta.
\end{equation}
Therefore, if
\begin{equation}
m\geq \frac{4}{\Gamma_{min}^2}\log \frac{N^2}{\delta}
\end{equation}
then with probability at least $1-\delta$,
\[\mu_\A\leq 2\max_{k,\ell} {\Gamma_{j\ell}}.\]
For stable recovery, we enforce that
\begin{equation}
\mu_\A\leq 2\max_{k,\ell} {\Gamma_{j\ell}}\leq 2\left(\frac{1}{\sqrt{2\gamma^2\sigma^2+1}}\right)^{q}\leq \frac{4}{\sqrt{41}(2s-1)}.
\end{equation}
This implies the following uncertainty principle 
\begin{equation}
\gamma^2\sigma^2\geq \frac{1}{2}\left(\left(\frac{\sqrt{41}(2s-1)}{2}\right)^{\frac{2}{q}}-1\right).
\end{equation}
To establish a suitable lower bound for $m$, we impose
\begin{equation}
\frac{4}{\Gamma_{min}^2}\log \frac{N^2}{\delta}= 
 4(\gamma^2\sigma^2+1)^{2q}\log \frac{N^2}{\delta} 
 \leq m.
\end{equation}
 When $q\geq d/2$, we have that $|\mathcal{G}|$ ranges between $2q-d$ and $q$.  Then analogously
\[
\left(\frac{1}{\sqrt{2\gamma^2\sigma^2+1}}\right)^{2q-d}\left(\frac{1}{\sqrt{\gamma^2\sigma^2+1}}\right)^{2d-2q}=:\Gamma_{\text{min}}\leq \Gamma_{j, \ell}\leq \left(\frac{1}{\sqrt{2\gamma^2\sigma^2+1}}\right)^q
\]
which implies the uncertainty principle 
\begin{equation}
\gamma^2\sigma^2\geq \frac{1}{2}\left(\left(\frac{\sqrt{41}(2s-1)}{2}\right)^{\frac{2}{q}}-1\right),
\end{equation}
and the lower bound for the number of samples
\begin{equation}
m\geq
 4(2\gamma^2\sigma^2+1)^{2q-d} (\gamma^2\sigma^2+1)^{2d-2q} \log \frac{N^2}{\delta}.
\end{equation}
This concludes the proof of \Cref{lem:coherence_analysis}.
\end{proof}


\section{Proofs for the Generalization Results}\label{sec:general_proofs}
To bound the generalization error, we will use the following two lemmata.

\newtheorem*{lemma1*}{Lemma 1}
\begin{lemma1*}[\textbf{Generalization Error, Term 1}]
Fix the confidence parameter $\delta > 0$ and accuracy parameter $\epsilon > 0$.  Recall the setting of \Cref{alg:B} and suppose $f \in \F(\phi,\rho)$ where $\phi(\x;\w) = \exp(i \langle \x,\w \rangle)$. The data samples $\x_k$ have probability measure  $\mu(\x)$ and weights $\w_j$ are sampled using the probability density $\rho(\w)$.
Consider the random feature approximation
\begin{equation}
f^{\star}(\x) := \sum_{j =1}^N c^{\star}_j \, \exp({i \langle \x, \w_j \rangle}), \quad \text{where} \ \ c^{\star}_j := \frac{\alpha(\w_j)}{N\rho(\w_j)}.
\end{equation}
If the number of features $N$ satisfies the bound
\begin{equation}
N \geq \frac{1}{\epsilon^{2}}\, \left(1 + \sqrt{2 \log\left(\frac{1}{\delta}\right)}\right)^2,
\end{equation}
then, with probability at least $1-\delta$ with respect to the draw of the weights ${\w_j}$ the following holds
\begin{equation}
\sqrt{\int_{\R^d} | f(\x) - f^{\star}(\x) |^2\, \d\mu }\leq \epsilon \|f\|_\rho.
\end{equation}
\end{lemma1*}
\begin{proof}
The proof follows similar arguments to those found in \cite{rahimi2008uniform, rahimi2008weighted}. The coefficients are bounded by $|c_j^\star|\leq \dfrac{\|f\|_\rho}{N}$ and, for fixed $\x$, $\E_{\w}[f^\star(\x)] = f(\x)$. Define the random variable 
\begin{align}
   v(\w_1,\dots,\w_N) &:= \|f - f^{\star}\|_{L^2(d\mu)} \\&= \|\mathbb{E}_\omega \left[f^{\star}\right] - f^{\star}\|_{L^2(d\mu)} 
= \left(\int_{\R^d} | \mathbb{E}_\omega \left[f^{\star}(\x)\right] - f^{\star}(\x) |^2 \d\mu \right)^{\frac{1}{2}}. 
\end{align}

To apply McDiarmid's inequality,  we show that $v$ is stable to perturbation. In particular, let $f^{\star}$ be the random feature approximation using random weights $(\w_1,\dots,{\w_{k}},\dots,\w_N)$ and let $\tilde{f}^{\star}$ be the random feature approximation using random weights $(\w_1,\dots,\tilde{\w}_k,\dots,\w_N)$, then 
\begin{equation}
\begin{aligned}
|v(\w_1,\dots,{\w_{k}},\dots,\w_N)&-v(\w_1,\dots,\tilde{\w}_k,\dots,\w_N)| \leq \left\|f^{\star}-\tilde{f}^{\star}\right\|_{L^2(d\mu)}\\
 &= \left\|c^{\star}_k \, \exp({i \langle \cdot, {\w_{k}} \rangle})-\tilde{c}^{\star}_k \, \exp({i \langle \cdot, \tilde{\w}_k \rangle})\right\|_{L^2(d\mu)}\\
  &= \frac{1}{N}\left\|\frac{\alpha({\w_{k}})}{\rho({\w_{k}})} \, \exp({i \langle \cdot, \w_k \rangle)}-\frac{\alpha(\tilde{\w}_k)}{\rho(\tilde{\w}_k)} \, \exp({i \langle \cdot, \tilde{\w}_k \rangle})\right\|_{L^2(d\mu)}\\
  &\leq \frac{1}{N} \sup_{\x}  \left|\frac{\alpha(\w_k)}{\rho(\w_k)}\, \exp({i \langle \x, \w_k \rangle)}-\frac{\alpha(\tilde{\w}_k)}{\rho(\tilde{\w}_k)}\, \exp({i \langle \x, \tilde{\w}_k \rangle)}\right| \\
  &\leq \frac{2\|f\|_\rho}{N}=:\Delta_v,
\end{aligned}
\end{equation}
where we used the triangle inequality for $\| \cdot \|_{L^2(d\mu)}$ in the first line, H\"{o}lder's inequality in the fourth line, and the uniform bound $\left|\exp({i \langle \x, \w \rangle)}\right|=1$ in the fifth line. 

To estimate the expectation of $v$, we bound the expectation of the second moment \cite{rahimi2008weighted}. By noting that the variance of an average of i.i.d. random variables is the average of the variances of each variable and by using the relation between the variance and the un-centered second moment, we have that
\begin{equation}
\begin{aligned}
  \E_{\w}[v^2] &=   \E_{\w} \|\mathbb{E}_\w \left[f^{\star}\right] - f^{\star}\|^2_{L^2(\d\mu)}\\
  &=\frac{1}{N} \left(\E_{\w} \left\| \frac{\alpha(\w)}{\rho(\w) }\exp({i \langle \cdot, \w \rangle)}\right\|^2_{L^2(\d\mu)} -\left\|\E_{\w}  \left[\frac{\alpha(\w)}{\rho(\w) }\exp({i \langle \cdot, \w \rangle)} \right]\right\|^2_{L^2(\d\mu)}  \right)\\
   &\leq \frac{1}{N} \, \E_{\w} \left| \frac{\alpha(\w)}{\rho(\w) }\right|^2 \\
  &\leq\frac{\|f\|^2_\rho}{N}.
\end{aligned}
\end{equation}
By Jensen's inequality, the expectation of $v$ is bounded by:
\begin{equation}
\begin{aligned}
  \E_{\w}[v]  \leq \left(  \E_{\w}[v^2] \right)^{\frac{1}{2}} \leq \frac{\|f\|_\rho}{\sqrt{N}}.
\end{aligned}
\end{equation}
Applying McDiarmid's concentration inequality yields
\begin{equation}
\P\left(v\geq \frac{\|f\|_\rho}{\sqrt{N}}+t\right)\leq\P(v\geq \E[v]+t)\leq \exp\left(-\frac{2t^2}{N \Delta_v^2}\right).
\end{equation}
Setting $t$ and $N$ to
\begin{equation}
t = \|f\|_\rho\sqrt{\frac{2}{N}\log\left(\frac{1}{\delta}\right)},
\end{equation}
and 
\begin{equation}
N \geq \frac{1}{\epsilon^{2}}\, \left(1 + \sqrt{2 \log\left(\frac{1}{\delta}\right)}\right)^2
\end{equation}
enforces that $M_v+t \leq \epsilon \|f\|_\rho$ and that the probability of failure is less than $\delta$. This completes the proof.
\end{proof}

\newtheorem*{lemma2*}{Lemma 2}
\begin{lemma2*}[\textbf{Generalization Error, Term 2}]
Let $f \in \F(\phi,\rho)$, where the basis function is $\phi(\x; \w)=\exp(i\langle \x, \w\rangle )$. For a fixed $\gamma$ and $q$, consider a set of data samples $\x_1,\dots, \x_m\sim \mathcal{N}(\mathbf{0}, \gamma^2 \I_d)$ with $\mu(\x)$ denoting the associated probability measure and weights $\w_1,\dots, \w_N$ drawn from $\mathcal N(\mathbf{0}, \sigma^2 \I_d)$. Assume that the noise is bounded by $E=2\nu$ or that the noise terms $e_j$ are drawn i.i.d. from $\N(0,\nu^2)$.
Let $\A\in \C^{m\times N}$ denote the associated  random feature matrix where $a_{k, j}=\phi(\x_k;\w_j)$. Let $f^\sharp$ be defined from \Cref{alg:B} and \Cref{eq:BP} with $\eta=\sqrt{2(\epsilon^2\|f\|_\rho^2+E^2)}$ and with the additional pruning step
\begin{equation}\nonumber
f^\sharp(\x) := \sum_{j\in\S^\sharp} \c^\sharp_j\, \phi(\x;\w_j),
\end{equation}
where $\S^\sharp$ is the support set of the $s$ largest (in magnitude) coefficients of $\c^\sharp$. Let the random feature approximation $f^\star$ be defined as
\begin{equation}
f^{\star}(\x) := \sum_{j =1}^N \c^{\star}_j \, \exp({i \langle \x, \w_j  \rangle)},
\end{equation}
where
 \begin{equation}
    \c^\star=\frac{1}{N}\, \left[\frac{\alpha(\w_1)}{\rho(\w_1)}, \cdots,\frac{\alpha(\w_N)}{\rho(\w_N)}\right]^T.
\end{equation}

For a given $s$, if the feature parameters $\sigma$ and $N$, the confidence $\delta$, and the accuracy $\epsilon$  are chosen so that the following conditions hold:
\begin{align*}
&\gamma^2\sigma^2\geq \frac{1}{2}\left(\left(\frac{\sqrt{41}(2s-1)}{2}\right)^{\frac{2}{d}}-1\right),\\
&N =  \frac{4}{\epsilon^2}\left(1 + 4 \gamma\sigma  d \sqrt{1+\sqrt{\frac{12}{d} \log\frac{m}{\delta}}}+ \sqrt{\frac{1}{2}\log\left(\frac{1}{\delta}\right)}\right)^2,\\
&m \geq 4(2\gamma^2\sigma^2+1)^{d} \log \frac{N^2}{\delta},
\end{align*}
then, with probability at least $1-4\delta$ the following error bound holds:
\begin{equation}
\begin{aligned}
\sqrt{\int_{\R^d} | f^\#(\x) - f^{\star}(\x) |^2 \, \d\mu }& \leq C \left(1+\, N^{\frac{1}{2}} \, s^{-\frac{1}{2}} \, m^{-\frac{1}{4}} \log^{1/4} \left(\frac{1}{\delta}\right)\right) \, \kappa_{s,1}(\c^\star)\\
&+  C'  \left( 1+ N^{\frac{1}{2}}m^{-\frac{1}{4}}\, \log^{1/4} \left(\frac{1}{\delta}\right)\right) \, \sqrt{\epsilon^2 \, \|f\|_\rho^2+4\nu^2},
\end{aligned}
\end{equation}
where $C,C'>0$ are constants.
\end{lemma2*}

\begin{proof}[Proof]
For simplicity, the coefficients $\c^\sharp$ are redefined to be zero outside of $\S^\sharp$.

To bound the generalization error, we will use McDiarmid's inequality. Define the random variable 
\begin{equation}
\begin{aligned}
v(\bf{z}_1,\dots,\bf{z}_m) &:= \int_{\R^d} | f^\#(\x) - f^{\star}(\x) |^2 \, \d\mu - \frac{1}{m} \sum_{k=1}^m | f^\#({\bf{z}}_k) - f^{\star}({\bf{z}}_k) |^2 \\
&=  \E_{\bf{z}}\left[ \frac{1}{m} \sum_{k=1}^m | f^\#({\bf{z}}_k) - f^{\star}({\bf{z}}_k)|^2\right] - \frac{1}{m} \sum_{k=1}^m | f^\#({\bf{z}}_k) - f^{\star}({\bf{z}}_k) |^2,\\
\end{aligned}
\end{equation}
where the i.i.d. random variables $\{{\bf{z}}_k\}_{k=1}^m$ are drawn from  $\mu$ (independent of the training samples $\x_k$) and noting that 
 \[\E_{{\bf{z}}}\left[ \frac{1}{m} \sum_{k=1}^m | f^\#({\bf{z}}_k) - f^{\star}({\bf{z}}_k)|^2\right] = \int_{\R^d} | f^\#(\x) - f^{\star}(\x) |^2 \, \d\mu,\]
 and thus the expectation of $v$ is zero. The points ${\bf{z}}_k$ are i.i.d. random samples and independent of $\x_k$'s, and thus independent of the coefficients. We choose $m$ total points in order to utilize the same coherence result for a random matrix depending on $\x_k$'s or ${\bf{z}}_k$'s.

To apply McDiarmid's inequality, we first show that $v$ is stable under a perturbation of any one of its coordinates. Perturbing just the $k\ts{th}$ coordinate leads to 
\begin{equation}
\begin{aligned}
|v({\bf{z}}_k)-v(\tilde{{\bf{z}}}_k)| &= \frac{1}{m} \left| | f^\#({\bf{z}}_k) - f^{\star}({\bf{z}}_k)|^2 -
| f^\#(\tilde{{\bf{z}}}_k) - f^{\star}(\tilde{{\bf{z}}}_k)|^2\right|
\end{aligned}
\end{equation}
Then for any ${\bf{z}}$ we have the uniform bound
\begin{equation}
\begin{aligned}
 |f^\#({\bf{z}}) - f^{\star}({\bf{z}})|^2&= \left|\,  \left[\phi({\bf{z}};\w_1), \ldots, \phi({\bf{z}};\w_N)\right] \left(\c^{\star}-\c^\sharp\right) \, \right|^2\\
& \leq  N \left\|\c^{\star}-\c^\sharp \, \right\|_2^2,
\end{aligned}
\end{equation}
by the Cauchy-Schwarz inequality. To bound $\|\c^{\star}-\c^\sharp\|_2$, we use \Cref{lem:l1_book}, which requires a bound on the stability parameter $\eta$.

 To determine the value of $\eta$ used for constructing $f^\sharp$, we bound the $\ell^2$ error between $\mathbf{y}$ and $\A \c^{\star}$
\begin{equation}
\begin{aligned}
\|\mathbf{y} - \A\c^{\star}\|^2 
&= \sum_{k = 1}^m (f(\x_k)-f^\star(\x_k) +e_k)^2  \\
&\leq 2\left(\sum_{k = 1}^m (f(\x_k)-f^\star(\x_k))^2+\sum_{j = k}^m e_k^2\right) \\
&= 2\left(\sum_{k= 1}^m (f(\x_k)-f^\star(\x_k))^2+ \|\e\|^2_{2}\right)\\
\end{aligned}
\end{equation}
where $\e:=[e_1, \ldots, e_m]^T\in \R^m$ is the noise vector. The assumption is that either $\|\e\|^2_{2} \leq 4\nu^2 m$ or that $\e$ is a random vector with i.i.d. elements drawn from $\N(0,\nu^2)$. In the second case, with probability at least $1-\delta$, the norm is bounded by  $\|\e\|^2_{2} \leq 4\nu^2 m$, as long as $m\geq 2 \log \left(\frac{1}{\delta}\right)$ which always holds by assumption, thus (by \Cref{errorboundlemma1})

\begin{equation}
\begin{aligned}
\|\mathbf{y} - \A\c^{\star}\|^2 
&\leq 2\left(\sum_{k= 1}^m (f(\x_k)-f^\star(\x_k))^2+4\nu^2 m\right)\\
&\leq 2m\left(\, \sup_{\| \x \|_2 \leq R}| f({\x})-f^\star({\x})|^2+4\nu^2\right)\\
&\leq 2m \left( \epsilon^2 \, \|f\|_\rho^2+4 \nu^2\right)\\
&=: m\eta^2
\end{aligned}
\end{equation}
as long as $N$ satisfies 
\begin{equation}
\begin{aligned}
N\geq  \frac{1}{\epsilon^2}\left(1 + 4 \gamma\sigma  d \sqrt{1+\sqrt{\frac{12}{d} \log\frac{m}{\delta}}}+ \sqrt{\frac{1}{2}\log\left(\frac{1}{\delta}\right)}\right)^2,
\end{aligned}
\end{equation}
where $R$ is replaced by the bound from \Cref{lem:gaussian_domain}.

Applying \Cref{lem:coherence_analysis} and \Cref{lem:thres_l1}, the $\ell^2$ error on the coefficients is bounded by
\begin{equation}
\begin{aligned}
\|\c^{\star}-\c^\sharp\|_2\leq  C' \, \frac{\kappa_{s,1}(\c^\star)}{\sqrt{s}}+ \sqrt{2} C\ \sqrt{\epsilon^2 \, \|f\|_\rho^2+{4\nu^2}}+4\kappa_{s,2}(\c^\star),
\end{aligned}
\end{equation}
for constants $C',C>0$ as long as
\begin{align*}
&\gamma^2\sigma^2\geq \frac{1}{2}\left(\left(\frac{\sqrt{41}(2s-1)}{2}\right)^{\frac{2}{d}}-1\right),\\
&m \geq 4(2\gamma^2\sigma^2+1)^{d}\log \frac{N^2}{\delta},
\end{align*}
holds. Note that by assumption $N^{-\frac{1}{2}}\leq \epsilon$ and 
\begin{equation}
\begin{aligned}\label{Kappa2}
\kappa_{s,2}(\c^\star) = \sqrt{\sum_{j \not\in \S^\star} \left| \c_j^\star\right|^2} \leq \frac{\sqrt{N-s}}{N} \|f\|_{\rho} \leq N^{-\frac{1}{2}} \|f\|_{\rho} 
\leq \epsilon\|f\|_{\rho}
\end{aligned}
\end{equation}
and thus after redefining $C$
\begin{equation}
\begin{aligned}\label{eqn:new_inequality}
\|\c^{\star}-\c^\sharp\|_2\leq  C' \, \frac{\kappa_{s,1}(\c^\star)}{\sqrt{s}}+ C\ \sqrt{\epsilon^2 \, \|f\|_\rho^2+{4\nu^2}}.
\end{aligned}
\end{equation}
Therefore, the difference in $v$ is bounded by 
\begin{equation}
\begin{aligned}
|v({{\bf{z}}_{k}})-v(\tilde{{\bf{z}}}_k)| &\leq \frac{1}{m} \left| | f^\#({\bf{z}}_k) - f^{\star}({\bf{z}}_k)|^2 -
| f^\#(\tilde{{\bf{z}}}_k) - f^{\star}(\tilde{{\bf{z}}}_k)|^2\right|\\
&\leq \frac{2N}{m} \|\c^\star-\c^\sharp\|^2_2\\
&= \frac{2N}{m} \left( C' \, \frac{\kappa_{s,1}(\c^\star)}{\sqrt{s}}+  C\ \sqrt{\epsilon^2 \, \|f\|_\rho^2+{4\nu^2}}\right)^2\\
&=:\Delta_v
\end{aligned}
\end{equation}
Applying McDiarmid's inequality
$\P( \E[v]-v\geq t)\leq \exp(-\frac{2t^2}{m \Delta_v^2})$,  yields:
\begin{equation}
\begin{aligned}
t &=\Delta_v \, \sqrt{\frac{m}{2}\, \log\left(\frac{1}{\delta}\right)}\\
&=\frac{\sqrt{2}N}{\sqrt{m}} \left( C' \, \frac{\kappa_{s,1}(\c^\star)}{\sqrt{s}}+ C\ \sqrt{\epsilon^2 \, \|f\|_\rho^2+{4\nu^2}}\right)^2 \, \sqrt{\log\left(\frac{1}{\delta}\right)}
\end{aligned}
\end{equation}
and thus
\begin{equation}
\begin{aligned}
&\int_{\R^d} | f^\#(\x) - f^{\star}(\x) |^2 \, \d\mu  \leq  \frac{1}{m} \sum_{k=1}^m | f^\#({\bf{z}}_k) - f^{\star}({\bf{z}}_k) |^2\\
&\quad+ \frac{N}{\sqrt{m}} \left( C' \, \frac{\kappa_{s,1}(\c^\star)}{\sqrt{s}}+ C\ \sqrt{\epsilon^2 \, \|f\|_\rho^2+{4\nu^2}}\right)^2\, \sqrt{\log\left(\frac{1}{\delta}\right)}\\
\end{aligned}
\end{equation}
with probability exceeding $1-\delta$ (after rescaling the constants). Therefore,
\begin{equation}
\begin{aligned}
&\sqrt{\int_{\R^d} | f^\#(\x) - f^{\star}(\x) |^2 \, \d\mu} \leq  m^{-\frac{1}{2}}\sqrt{\sum_{k=1}^m | f^\#({\bf{z}}_k) - f^{\star}({\bf{z}}_k) |^2} \\
&+ N^{\frac{1}{2}}m^{-\frac{1}{4}}\, \left( C' \, \frac{\kappa_{s,1}(\c^\star)}{\sqrt{s}}+ C\ \sqrt{\epsilon^2 \, \|f\|_\rho^2+4\nu^2}\right) \, \left(\log\left(\frac{1}{\delta}\right)\right)^{1/4}.
\end{aligned}
\end{equation}

Define the approximation $f^\star_s$ as:
\begin{equation}
f_{s}^{\star}(\x) := \sum_{j \in\S^\star} \c^{\star}_j \, \exp({i \langle \x, \w_j \rangle}),
\end{equation}
where $\S^\star$ is the support set of the $s$ largest (in magnitude) coefficients of $\c^\star$. Let $\tilde{A} \in \C^{m\times N}$ denote the associated random feature matrix with $\tilde{a}_{k, j}=\phi({\bf{z}}_k;\w_j)$, then
\begin{equation}
\begin{aligned}
&m^{-\frac{1}{2}}\sqrt{\sum_{k=1}^m | f^\#({\bf{z}}_k) - f^{\star}({\bf{z}}_k) |^2} \\
&\leq m^{-\frac{1}{2}}\sqrt{\sum_{k=1}^m | f^\#({\bf{z}}_k) - f_s^{\star}({\bf{z}}_k) |^2}+m^{-\frac{1}{2}}\sqrt{\sum_{k=1}^m | f_s^{\star}({\bf{z}}_k) - f^{\star}({\bf{z}}_k) |^2}\\
&\leq m^{-\frac{1}{2}}\sqrt{\sum_{k=1}^m | f^\#({\bf{z}}_k) - f_s^{\star}({\bf{z}}_k) |^2}+m^{-\frac{1}{2}}\sqrt{\sum_{k=1}^m \left| \sum_{j \not\in \S} \c^{\star}_j \, \exp({i \langle {\bf{z}}_k, \w_j  \rangle)} \right|^2}\\
&\leq \sqrt{\sum_{k=1}^m | m^{-\frac{1}{2}}(f^\#({\bf{z}}_k) - f_s^{\star}({\bf{z}}_k)) |^2}+m^{-\frac{1}{2}}\sqrt{m \left( \sum_{j \not\in \S} |\c^{\star}_j |\,\right)^2}\\
&= \sqrt{\sum_{k=1}^m | m^{-\frac{1}{2}}(f^\#({\bf{z}}_k) - f_s^{\star}({\bf{z}}_k)) |^2}+\kappa_{s,1}(\c^\star)\\
&=   \| m^{-\frac{1}{2}}\tilde{A}( \c^{\sharp} - \c^{\star}_s )\|_2 +\kappa_{s,1}(\c^\star)\\
&\leq \left(1+\frac{4}{\sqrt{41}}\right)^{\frac{1}{2}} \left\|\c^\sharp-\c_s^* \, \right\|_2+\kappa_{s,1}(\c^\star)\\
&\leq 2 \left\|\c^\sharp-\c^\star \, \right\|_2+2 \left\|\c_s^\star-\c^\star \, \right\|_2+\kappa_{s,1}(\c^\star)\\
&\leq 2 \left\|\c^\sharp-\c^\star \, \right\|_2+2\kappa_{s,2}(\c^\star)+\kappa_{s,1}(\c^\star)\\
\end{aligned}
\end{equation}
where we used the $2s$-RIP condition for $\tilde{A}$ (which is guaranteed by  \Cref{lem:coherence_analysis}) and the fact that $\c^{\sharp} - \c^{\star}_s$ is $2s$-sparse.

Altogether we have
 \begin{equation}
\begin{aligned}
&\sqrt{\int_{\R^d} | f^\#(\x) - f^{\star}(\x) |^2 \, \d\mu } \\
&\leq2 \left( C' \, \frac{\kappa_{s,1}(\c^\star)}{\sqrt{s}}+  C\ \sqrt{\epsilon^2 \, \|f\|_\rho^2+4\nu^2}\right) +2\kappa_{s,2}(\c^\star)+\kappa_{s,1}(\c^\star)  \\
&\quad \quad + N^{\frac{1}{2}}m^{-\frac{1}{4}}\, \left( C' \, \frac{\kappa_{s,1}(\c^\star)}{\sqrt{s}}+  C\ \sqrt{\epsilon^2 \, \|f\|_\rho^2+4\nu^2}\right) \, \left(\log\left(\frac{1}{\delta}\right)\right)^{1/4}\\
&\leq \left(2 C's^{-\frac{1}{2}} +1+ C'\, N^{\frac{1}{2}} \, s^{-\frac{1}{2}} \, m^{-\frac{1}{4}} \, \log^{1/4}\left(\frac{1}{\delta}\right)\right) \, \kappa_{s,1}(\c^\star)\\
&\quad \quad + C  \left( 2+ N^{\frac{1}{2}}m^{-\frac{1}{4}}\, \log^{1/4}\left(\frac{1}{\delta}\right)\right) \, \sqrt{\epsilon^2 \, \|f\|_\rho^2+4\nu^2}+ 2\kappa_{s,2}(\c^\star),\\
&\leq C'\left(1+\, N^{\frac{1}{2}} \, s^{-\frac{1}{2}} \, m^{-\frac{1}{4}} \, \log^{1/4}\left(\frac{1}{\delta}\right)\right) \, \kappa_{s,1}(\c^\star)\\
&\quad \quad + C  \left( 1+ N^{\frac{1}{2}}m^{-\frac{1}{4}}\, \log^{1/4}\left(\frac{1}{\delta}\right)\right) \, \sqrt{\epsilon^2 \, \|f\|_\rho^2+4\nu^2},
\end{aligned}
\end{equation}
where we used \Cref{Kappa2} and redefined $C$ and $C'$. This concludes the proof.

\end{proof}

To establish \Cref{thm:low_order}, first note that $f$ is an order-$q$ function of at most $K$ terms, it  can be written as:
\begin{equation}
f(x_1,\dots, x_d) = \frac{1}{K}\sum_{j=1}^K g_j(x_{j_1}, \dots, x_{j_q}),
\end{equation}
for a fixed $q$. For each term, define $g_j^\star$ as
\begin{equation}\label{defcstarjl}
g_{j}^{\star}(\x|_{\S_j}) = \sum_{\ell=1}^N \, \tilde{\c}_{j,\ell}^{\star}\ \exp({ i \langle\x|_{\S_j}, \w_{\ell}|_{\S_j}\rangle}), \quad \text{where} \ \ \tilde{\c}_{j,\ell}^{\star} =\begin{cases}
\frac{\alpha_j(\w_\ell)}{n\, \rho(\w_\ell)}, &\text{if}  \ \supp(\w_\ell)= \S_j\\
0, &\text{otherwise}. 
\end{cases}
\end{equation}
We define  $f^\star$ as
\begin{equation}
\begin{aligned}
f^{\star}(\x) := \frac{1}{K}\sum_{j=1}^K  g_{j}^{\star}(\x|_{\S_j}) &= \frac{1}{K}\sum_{j=1}^K \sum_{\ell=1}^N \,  \tilde{\c}_{j,\ell}^{\star}\ \exp( i \langle\x|_{\S_j}, \w_{\ell}|_{\S_j}\rangle)\\
&=\frac{1}{K}\sum_{j=1}^K \sum_{\ell=1}^N \  \tilde{\c}_{j,\ell}^{\star}\ \exp( i \langle\x, \w_{\ell}\rangle),\\
&= \sum_{\ell=1}^N \,  \left(\frac{1}{K}\sum_{j=1}^K \tilde{\c}_{j,\ell}^{\star} \right) \ \exp( i \langle\x, \w_{\ell}\rangle)\\
&= \sum_{\ell=1}^N \,  \tilde{\c}_\ell^{\star}\, \exp( i \langle\x, \w_{\ell}\rangle),\\
\end{aligned}
\end{equation}
where in the second line we use \Cref{defcstarjl} and we define $\tilde{\c}_\ell^{\star}:=\frac{1}{K}\sum\limits_{j=1}^K \tilde{\c}_{j,\ell}^{\star}$.  
For each term $g_j$, only $n$ out of the $N$ features are active, so $\c^\star$ is $nK$-sparse if only $K$ functions $g_j$ are nonzero.
Since there are $K$ such terms, by applying the union bound, if
\begin{equation}
n \geq \frac{4 }{\epsilon^2}\left(1 + 4 R \sigma \sqrt{q} + \sqrt{\frac{1}{2}\log\left(\frac{K}{\delta}\right)}\right)^2,
\end{equation}
then 
\begin{equation}
\sup_{\| \x|_{\S_j} \| \leq R} | g_j(\x|_{\S_j}) -  g_{j}^{\star}(\x|_{\S_j}) | \leq \epsilon  \|g_j\|_\rho,
\end{equation}
holds for each $j \in [K]$.
By the triangle inequality, we have 
\begin{equation}
\begin{aligned}
\sup_{\substack{\x\in \mathbb{R}^d:   
\| \x|_{\S_j} \| \leq R,\\ \; \forall j\in [K]}} | f(\x) - f^{\star}(\x) |
&\leq\dfrac{1}{K}\, \sup_{\substack{\x\in \mathbb{R}^d:   
\| \x|_{\S_j} \| \leq R, \\\; \forall j\in [K]}}\,  \sum\limits_{j=1}^K |g_j(\x|_{\S_j}) -  g_{j}^{\star}(\x|_{\S_j})| \\
&\leq\dfrac{1}{K} \sum\limits_{j=1}^K \left(\sup_{\x\in \mathbb{R}^d:   
\| \x|_{\S_j} \| \leq R}|g_j(\x|_{\S_j}) -  g_{j}^{\star}(\x|_{\S_j})|\right) \\
&\leq\dfrac{\epsilon}{K} \sum_{j=1}^K\|g_j\|_\rho\\
&\leq\epsilon\vertiii{f}.
\end{aligned}
\end{equation}
This will be used to define the stability parameter in the basis pursuit problem.

Next observe that 
\begin{align}
 \sqrt{\int_{\R^d} | f(\x) - f^{\star}(\x) |^2 \, \d\mu } &= 
 \left\| \frac{1}{K} \sum_{j=1}^K\left(   g_j(\x|_{S_j}) - g_j^{\star}(\x|_{S_j})  \right) \right\|_{L^2(d\mu)} \nonumber \\
  &\leq \frac{1}{K} \sum_{j=1}^K \left\| g_j(\x|_{S_j}) - g_j^{\star}(\x|_{S_j}) \right\|_{L^2(d\mu)} \nonumber \\
 &= \frac{1}{K} \sum_{j=1}^K \left\| g_j(\x|_{S_j}) - g_j^{\star}(\x|_{S_j}) \right\|_{L^2(d\tilde{\mu})},
\end{align}
where $\mu(\x)$ is the probability measure associated with the $d$-dimensional spherical Gaussian ${\cal N}({\bf 0}, \gamma^2 {\bf I}_d)$ and $\tilde{\mu}(\x)$ is the probability measure associated with the $q$-dimensional spherical Gaussian ${\cal N}({\bf 0}, \gamma^2 {\bf I}_q).$ Then we can apply \Cref{lem:error_term1} to each error term $\left\| g_j(\x|_{S_j}) - g_j^{\star}(\x|_{S_j}) \right\|_{L^2(d\tilde{\mu})}$ to get 
\begin{align}
 \sqrt{\int_{\R^d} | f(\x) - f^{\star}(\x) |^2 \, \d\mu } \leq \epsilon \vertiii{f},
\end{align}
which bounds the first error term.

Note that the proof of \Cref{lem:error_term2} holds with $\eta=\sqrt{2\left(\epsilon^2\vertiii{f}^2+E^2\right)}$,  and $N = n\, {\genfrac(){0pt}{2}{d}{q}}$, since 
\begin{equation}
\begin{aligned}\label{Kappa2}
\kappa_{s,2}(\tilde{\c}^\star) \leq \frac{\sqrt{\text{max}(nK-s,0)}}{n} \, \vertiii{f}\leq K^{\frac{1}{2}}\, n^{-\frac{1}{2}}\,\vertiii{f}  \leq \eta.
\end{aligned}
\end{equation}
By rescaling the term $\epsilon$ to $ \epsilon\, {\genfrac(){0pt}{2}{d}{q}}^{\frac{1}{2}}$ and assuming $\epsilon\, {\genfrac(){0pt}{2}{d}{q}}^{\frac{1}{2}}$ is sufficiently small, we conclude the proof.

\end{document}